\documentclass[twoside]{article}
\pdfoutput=1

\usepackage[preprint]{aistats2022}
\usepackage[utf8]{inputenc} 
\usepackage[T1]{fontenc}    
\usepackage{hyperref}       
\usepackage{url}            
\usepackage{booktabs}       
\usepackage{amsfonts}       
\usepackage{nicefrac}       
\usepackage{microtype}      
\usepackage{xcolor}         
\usepackage{amsmath}
\usepackage{amssymb}
\usepackage{amsthm}
\usepackage{algorithm}
\usepackage[noend]{algpseudocode}
\usepackage{makecell}
\usepackage{multicol}
\usepackage{multirow}
\usepackage{todonotes}
\usepackage{graphicx}
\usepackage{wrapfig}
\usepackage{caption}
\usepackage[acronym,smallcaps,nowarn,section,nogroupskip,nonumberlist]{glossaries}
\usepackage[nameinlink,capitalise]{cleveref}
\usepackage{listings}
\usepackage{natbib}
\usepackage{balance}

\usepackage{color}
\usepackage[normalem]{ulem}

\usepackage{amssymb}
\usepackage{pifont}
\usepackage{amsfonts}

\newcommand{\bmx}[0]{\begin{bmatrix}}
\newcommand{\emx}[0]{\end{bmatrix}}

\newcommand{\grad}[0]{\nabla}

\Crefname{algocf}{Algorithm}{Algorithms}
\crefname{algorithm}{Algorithm}{Algorithms}
\crefname{equation}{Equation}{Equations}
\crefname{figure}{Figure}{Figure}
\crefname{section}{§}{§§}
\Crefname{section}{§}{§§}
\crefformat{equation}{(#2#1#3)}

\newtheorem{theorem}{Theorem}

\newcommand{\expect}[2]{\mathop{\mathbb{E}}_{#2} \left[ {#1} \right]}

\newcommand{\stopgrad}[1]{\bot\left({#1}\right)}

\newcommand{\evaluate}[1]{\overrightarrow{#1}}
\newcommand{\resampled}[1]{\tilde{#1}}
\newcommand{\normalized}[1]{\overline{#1}}
\newcommand{\lineage}[1]{\tilde{#1}}
\newcommand{\Lineage}{L}
\newcommand{\sgw}{v}  
\newcommand{\Sgw}{V}  

\title{Differentiable Particle Filtering\\ without Modifying the Forward Pass}
\runningtitle{Differentiable Particle Filtering without Modifying the Forward Pass}
\author{%
  Adam \'Scibior\textsuperscript{1,2}\\
  \texttt{ascibior@cs.ubc.ca}\\
  \And
  Frank Wood\textsuperscript{1,2,3}\\
  \texttt{fwood@cs.ubc.ca}\\
  \AND
  \\
  \textsuperscript{1}University of British Columbia
  \quad
  \textsuperscript{2}Inverted AI
  \quad
  \textsuperscript{3}Mila
}
\runningauthor{Adam \'Scibior, Frank Wood}

\begin{document}
\twocolumn[

\aistatstitle{Differentiable Particle Filtering\\ without Modifying the Forward Pass}

\aistatsauthor{
    Adam \'Scibior\textsuperscript{1,2}\\
    \texttt{ascibior@cs.ubc.ca}\\
    \And
    Frank Wood\textsuperscript{1,2,3}\\
    \texttt{fwood@cs.ubc.ca}
    }

\aistatsaddress{
    \textsuperscript{1}Inverted AI
    \And
    \textsuperscript{2}University of British Columbia
    \And
    \textsuperscript{3}Mila
}
]

\begin{abstract}
    Particle filters are not compatible with automatic differentiation due to the presence of discrete resampling steps. While known estimators for the score function, based on Fisher's identity, can be computed using particle filters, up to this point they required manual implementation. In this paper we show that such estimators can be computed using automatic differentiation, after introducing a simple correction to the particle weights. This correction utilizes the stop-gradient operator and does not modify the particle filter operation on the forward pass, while also being cheap and easy to compute. Surprisingly, with the same correction automatic differentiation also produces good estimators for gradients of expectations under the posterior. We can therefore regard our method as a general recipe for making particle filters differentiable. We additionally show that it produces desired estimators for second-order derivatives and how to extend it to further reduce variance at the expense of additional computation.
\end{abstract}

\section{Introduction}

Sequential latent variable models can be learned by gradient ascent, applying automatic differentiation (AD) to the log-marginal likelihood estimators obtained with particle filters. Examples include robotic localization \citep{jonschkowski_differentiable_2018}, world-modeling in reinforcement learning \citep{igl2018deep}, and deep generative modeling of music and speech \citep{maddison_filtering_2017}. Unfortunately, the discrete resampling steps allowing particle filters to adaptively focus computational resources on promising paths cause problems for automatic differentiation, with the usual algorithms producing gradient estimators that suffer from excessive variance. In current practice, the dependence of the resampling probabilities on model parameters is typically ignored \citep{le_auto-encoding_2018, maddison_filtering_2017, naesseth_variational_2018}, introducing bias that was shown not to vanish asymptotically \citep{corenflos_differentiable_2021} and impede learning.

Several recent articles \citep{corenflos_differentiable_2021, zhu_towards_2020, karkus_particle_2018} have proposed continuous relaxations of the discrete resampling steps to avoid this problem, with different trade-offs and varying degrees of complexity, as discussed in detail in Section \ref{sec:diff-resampling}. 
At the same time, there exist consistent estimators of the gradient of log-marginal likelihood that can be obtained with usual discrete resampling \citep{poyiadjis_particle_2011}, but until now these could not be obtained by applying AD to the log-marginal likelihood estimator produced by a particle filter.

In this work we show that a simple and computationally inexpensive modification to an otherwise entirely standard ``basic'' particle filter implementation allows AD to produce the estimators derived by \citet{poyiadjis_particle_2011}. The modification, shown in Algorithm \ref{alg:pf}, can be thought of as stopping the gradients of weights flowing into the resampling distribution, instead directing them into the resulting weights. Crucially, it does not modify the estimators produced by the particle filter on the forward pass, only changing the gradients computed by AD. We also show that with our modification AD produces sensible, easy to interpret, and consistent estimators for gradients of expectations under the posterior, which are additionally unbiased if the estimator obtained on the forward pass is unbiased. We can therefore regard our modification as a general technique for making a particle filter differentiable, without modifying its behavior on the forward pass. While existing differentiable resampling mechanisms have their merits, the simplicity of our approach is unmatched, so we advocate trying it first before attempting more sophisticated alternatives.

Additionally, in Section \ref{sec:extensions} we show that our results extend to higher-order derivatives computed by repeated AD and that the same method applies to marginal particle filters \citep{klaas_toward_2012}, recovering the more expensive but lower variance estimator proposed by \citet{poyiadjis_particle_2011}. We also briefly discuss the problem of jointly learning a proposal distribution with the model, which is largely left for future work.

\section{Preliminaries}

\begin{algorithm*}
\caption{Forward pass of the standard particle filter (left), compared with our differentiable particle filter (right). The crucial modification required by our method is highlighted in yellow, where the $\bot$ denotes the stop-gradient operator. The application of stop-gradient inside the resampling function is typically redundant, because automatic differentiation libraries do not propagate gradients through samples from discrete distributions. We choose the symbol $\sgw$ to denote the weights of the differentiable particle filter, to distinguish them from the traditionally computed weights in the equations throughout the paper. Where $i$ is not bounded it denotes an operation applied across all $i \in 1:N$, and we use colons to index sequences, e.g. $x_{1:T} = (x_1, \dots, x_T)$.}
\label{alg:pf}
\begin{minipage}{0.45\textwidth}
\begin{algorithmic}
\Function{PF}{$p$, $q$, $N$, $T$}
  \State $x_0^i \sim p(x_0)$
  \State $w_0^i = \normalized{w}_0^i = \frac{1}{N}$
  \For{$t \in 1:T$}
    \If{resampling condition}
      \State $a^{1:N}_{t} \sim R(\normalized{w}^{1:N}_{t-1})$
      \State $\resampled{w}^i_{t} = \frac{1}{N} $
     \Else
       \State $a^i_t = i$
       \State $\resampled{w}^i_t = \normalized{w}^{i}_{t-1}$
    \EndIf
    \State $x^i_t \sim q_{\phi}(x_t | x^{a_t^i}_{t-1})$
    \State $w^i_t = \resampled{w}^i_{t-1} \frac{p_{\theta}(x^i_t, y_t | x^{a_t^i}_{t-1})}{q_{\phi}(x^i_t | x^{a_t^i}_{t-1})}$
    \State $W_t = \sum_{i=1}^N w^i_{t}$
    \State $\normalized{w}^i_t = w^i_t / W_t$
  \EndFor
  \State $\hat{Z}_{\text{PF}} = \prod_{t=1}^T W_t$
\EndFunction
\end{algorithmic}
\end{minipage}
\begin{minipage}{0.45\textwidth}
\begin{algorithmic}
\Function{DPF}{$p$, $q$, $N$, $T$}
  \State $x_0^i \sim p(x_0)$
  \State $\sgw_0^i = \normalized{\sgw}_0^i = \frac{1}{N}$
  \For{$t \in 1:T$}
    \If{resampling condition}
      \State $a^{1:N}_{t} \sim R(\stopgrad{\normalized{\sgw}^{1:N}_{t-1}})$
      \State $\resampled{\sgw}^i_{t} = \frac{1}{N} $ \colorbox{yellow}{$\displaystyle {\normalized{\sgw}^{a^i_{t}}_{t-1} / {\stopgrad{\normalized{\sgw}^{a^i_t}_{t-1}}}}$}
     \Else
       \State $a^i_t = i$
       \State $\resampled{\sgw}^i_t = \normalized{\sgw}^{i}_{t-1}$
    \EndIf
    \State $x^i_t \sim q_{\phi}(x_t | x^{a_t^i}_{t-1})$
    \State $\sgw^i_t = \resampled{\sgw}^i_{t-1} \frac{p_{\theta}(x^i_t, y_t | x^{a_t^i}_{t-1})}{q_{\phi}(x^i_t | x^{a_t^i}_{t-1})}$
    \State $\Sgw_t = \sum_{i=1}^N \sgw^i_{t}$
    \State $\normalized{\sgw}^i_t = \sgw^i_t / \Sgw_t$
  \EndFor
  \State $\hat{Z}_{\text{DPF}} = \prod_{t=1}^T \Sgw_t$
\EndFunction
\end{algorithmic}
\end{minipage}
\end{algorithm*}

We assume familiarity with automatic differentiation \citep{baydin_automatic_2018}, including its application to stochastic computations \citep{schulman_gradient_2016}, but provide some additional exposition to those and other topics in the supplementary materials.

\subsection{Particle filters} \label{sec:pf}

We assume familiarity with particle filters \citep{chopin_introduction_2020, doucet2009tutorial, doucet2001introduction}, using notation defined in Algorithm \ref{alg:pf}. Our results apply with any choice of the resampling scheme $R$ \citep{murray2016parallel}.

We assume that standard conditions for the existing convergence results \citep{moral_feynman-kac_2004, chopin_introduction_2020} are satisfied, specifically for consistency of estimators of the log-marginal likelihood and expectations under the posterior, and for unbiasedness of estimators of expectations under the unnormalized posterior. For score function estimation with particle filters, we use the estimators derived by \citet{poyiadjis_particle_2011} using Fisher's identity, in particular
\begin{align}
    \grad_{\theta} \log p_{\theta}(y_{1:T}) \approx \sum_{i=1}^N \normalized{w}_T^i \grad_{\theta} \log p_{\theta}(\lineage{x}_{1:T}^i, y_{1:T}) ,
    \label{eq:poyadjis}
\end{align}
where $\lineage{x}_{1:t}^i$ is the ancestral lineage of a particle $i$ at time $t$, satisfying the recursive equation $\lineage{x}_{1:t}^i = (\lineage{x}_{1:t-1}^{a_t^i}, x_{t}^i$).
Our key contribution is a construction that enables obtaining this estimator through automatic differentiation of the $\log \hat{Z}$ estimate from a particle filter, without affecting estimators produced on the forward pass.

\subsection{The stop-gradient operator} \label{sec:stop-gradient}

The stop-gradient operator is a standard component of automatic differentiation libraries, called \lstinline{stop_gradient} in Tensorflow and \lstinline{detach} in PyTorch, producing an expression that evaluates to its normal value on the forward pass but produces no gradient. In this section we introduce a minimal formalism required to reason about gradient estimators produced by AD from complex expressions involving stop-gradient.\footnote{Note that there doesn't exist a $\mathbb{R} \rightarrow \mathbb{R}$ function with the properties of stop-gradient, so extra care needs to be taken when manipulating expressions containing it.} Following \citet{foerster_dice_2018}, we denote the stop-gradient with a $\bot$, for example $\grad \stopgrad{x} = 0$.

We refer to expressions containing stop-gradient as \emph{surrogate expressions} and we need to distinguish between the results obtained by evaluating them on the forward pass and the expressions obtained by applying AD. Following \citet{van_krieken_storchastic_2021}, we denote evaluation of an expression $E$ with an overhead right arrow $\evaluate{E}$. Operationally, evaluation removes all instances of stop-gradient, provided that no gradient operator is acting on them. Gradients satisfy the usual chain rules, except that gradients of any expression wrapped in stop-gradient are zero. The key equations are:
\begin{align}
    \evaluate{f(E_1, \dots, E_n)} &= f(\evaluate{E_1}, \dots, \evaluate{E_2}) \\
    \evaluate{\stopgrad{E}} &= \evaluate{E} \\
    \grad{\stopgrad{E}} &= 0 \\
    \evaluate{\grad{E}} &= \grad{E} \quad \text{if} \ \bot \notin E  ,
\end{align}
where $f$ is some fixed differentiable function and $\bot \notin E$ means that stop-gradient is not applied anywhere in expression $E$, including its subexpressions. Crucially, in general $\grad \evaluate{E} \neq \evaluate{\grad E}$, and $\evaluate{E_1} = \evaluate{E_2}$ \emph{does not imply} $\evaluate{\grad E_1} = \evaluate{\grad E_2}$.

We use this calculus to derive formulas for estimators obtained by algorithms utilizing stop-gradient in automatic differentiation, which we generally accomplish by using the rules above, along with the chain rule for the gradient operator, to push the gradients inside, past any stop-gradient operators, and then evaluating the resulting expression. Note that in this calculus the gradient operator produces expressions rather than values, so that we can model repeated automatic differentiation. Expressions without stop-gradient can be equated with their usual interpretations. Throughout the paper we are careful to avoid using integrals of expressions involving stop-gradient, as we consider those undefined. We also present a short tutorial on using this calculus in the appendix.

Surrogate expressions are commonly used when computing gradients of computation involving non-reparameterizable random variables, which is typically implemented using surrogates of the form $f_{\theta}(x) + \stopgrad{f_{\theta}(x)} \log p_{\theta}(x)$. The resulting gradient estimator is known in the reinforcement learning literature as REINFORCE \citep{williams_simple_1992}, and more broadly referred to as the score function method or the likelihood ratio method \citep{fu_chapter_2006}. We refrain from using those terms to avoid confusion with the score function $\grad_{\theta} \log p_{\theta}(y_{1:T})$ and the likelihood $p_{\theta}(y_{1:T})$ associated with the sequential latent variable model. \citet{foerster_dice_2018} provide a method for constructing alternative surrogates, called DiCE, which produces the same gradient estimators but has the additional advantage of preserving the results of evaluating the forward pass, as well as producing unbiased estimators for derivatives of higher order. As an illustrative example, we use the calculus introduced above to compute the gradient estimator produced by DiCE
\begin{align}
    &\evaluate{\grad_{\theta} \left( \frac{p_{\theta}(x)}{\stopgrad{p_{\theta}(x)}} f_{\theta}(x) \right)} = \\
    &f_{\theta}(x) \evaluate{\frac{\grad_{\theta} p_{\theta}(x)}{\stopgrad{p_{\theta}(x)}}} + \evaluate{\frac{p_{\theta}(x)}{\stopgrad{p_{\theta}(x)}}} \grad_{\theta} f_{\theta}(x) = \nonumber \\
    &f_{\theta}(x) \grad_{\theta} \log p_{\theta}(x) + \grad_{\theta} f_{\theta}(x) . \nonumber
\end{align}
Since it produces the same gradient estimators as REINFORCE, a.k.a.~the score function method, a.k.a.~the likelihood ratio method, we use the name DiCE throughout the paper to encapsulate all those methods.

\section{Method}

In this section we derive and analyze the gradient estimators obtained by applying AD to two different versions of a particle filter, namely the standard implementation (PF) and our modified version (DPF), contrasted in Algorithm \ref{alg:pf}. To avoid further complicating notation, throughout the paper we assume that resampling is performed at each time step. These two particle filters are equivalent on the forward pass, in particular $\evaluate{w_t^i} = \evaluate{v_t^i}$ and $\evaluate{\hat{Z}_{\text{PF}}} = \evaluate{\hat{Z}_{\text{DPF}}}$, and similarly $\stopgrad{w_t^i} = \stopgrad{v_t^i}$ and $\stopgrad{\hat{Z}_{\text{PF}}} = \stopgrad{\hat{Z}_{\text{DPF}}}$. The distinction is only material when computing gradients of such expressions, when that is not the case we sometimes drop the subscript and simply write $\hat{Z}$ to avoid clutter. 

\subsection{Estimating the score function} \label{sec:score}

To provide a gentle exposition, we first consider simple importance sampling in a non-sequential latent variable model, without any resampling or ancestral variables. A simple way to obtain score function estimates is to apply automatic differentiation to a program that estimates the log-marginal likelihood with importance sampling, as is done in IWAE \citep{burda_importance_2016}. The forward pass of such a program defines the following objective
\begin{align}
    \mathcal{L}_{\text{IWAE}} =
    \expect{\log \frac{1}{N} \sum_{i=1}^N \frac{p_{\theta}(x^i, y)}{q_{\phi}(x^i)}}{x^i \sim q_{\phi}(x)} .
\end{align}
Automatic differentiation then computes gradients with respect to the model parameters $\theta$ as
\begin{align}
    \grad_{\theta} \mathcal{L}_{\text{IWAE}} =
    \expect{\sum_{i=1}^N \normalized{w}^i \grad_{\theta} \log p_{\theta}(x^i, y)}{x^i \sim q_{\phi}(x)} ,
    \label{eq:iwae-grad}
\end{align}
where $w^i = \frac{p_{\theta}(x^i, y)}{q_{\phi}(x^i)}$ and $\normalized{w}^i = w^i / \sum_{j=1}^N w^j$.

This is exactly the same estimator we would obtain by using Fisher's identity with a self-normalizing importance sampler to approximate the posterior. However, the two methods are no longer equivalent if the proposal distribution also depends on model parameters, such as when using the prior $p_{\theta}(x)$ as the proposal. Applying Fisher's identity produces the same estimator as Eq. \ref{eq:iwae-grad} whether $q$ depends on $\theta$ or not. With automatic differentiation, since the sampling distribution depends on $\theta$ and assuming it can not be reparameterized, we use DiCE and obtain the following gradient estimator

\begin{align}
    \log \hat{Z} \ \sum_{i=1}^N \grad_{\theta} \log q_{\theta}(x^i) +
        \sum_{i=1}^N \normalized{w}^i \grad_{\theta} \log \frac{p_{\theta}(x^i, y)}{q_{\theta}(x^i)} \label{eq:iwae-reinforce} ,
\end{align}
where $\hat{Z} = \sum_{i=1}^N w^i$.
The first term inside the expectation is the problematic one, and it tends to either be dropped, introducing bias, or necessitates the use of baselines or other variance reduction methods. When gradients of $q$ are stopped, the first term disappears, along with the denominator in the second one, recovering the estimator from Eq. \ref{eq:iwae-grad}. We can show this formally using the calculus introduced in Section \ref{sec:stop-gradient}. We introduce a stop-gradient weight $v^i = {p_{\theta}(x^i, y)} / {\stopgrad{q_{\theta}(x^i)}}$ to distinguish it from  $w^i = {p_{\theta}(x^i, y)} / {q_{\theta}(x^i)}$, noting that they evaluate to the same value on the forward pass, that is $\evaluate{v^i} = w^i$. Then
\begin{align}
    & \evaluate{\grad_{\theta} \log \frac{1}{N} \sum_{i=1}^N \frac{p_{\theta}(x^i, y)}{\stopgrad{q_{\theta}(x^i)}}} = \evaluate{\grad_{\theta} \log \frac{1}{N} \sum_{i=1}^N v^i} \\
    & \evaluate{\frac{\frac{1}{N} \sum_{i=1}^N \grad_{\theta} v^i}{\frac{1}{N} \sum_{i=1}^N v^i}} =
    \evaluate{\sum_{i=1}^N \normalized{v}^i \frac{\grad_{\theta} v^i}{v^i}} = \\
    &\evaluate{\sum_{i=1}^N \normalized{\sgw}^i \grad_{\theta} \log \sgw^i} =
    \sum_{i=1}^N \normalized{w}^i \grad_{\theta} \log p_{\theta}(x^i, y) \label{eq:fisher-autodiff} .
\end{align}
We emphasize that the use of this estimator is justified by Fisher's identity, independently of any stop-gradient considerations required to obtain it through automatic differentiation. This lets us access known theoretical results, in particular implying consistency under standard assumptions.

Those results suggests that \emph{when the proposal distribution depends on the model parameters}, stop-gradient should be applied to it, in particular if it can not be reparameterized. In IWAE, or sequential importance sampling, this problem is usually avoided by having separate parameters for the proposal distribution. In particle filters, however, even if the proposal distribution uses separate parameters, the computation involves discrete ancestral indices, which are sampled according to weights, which in turn depend on the model parameters. Stopping those gradients changes the resampling distribution, which requires multiplying the post-resampling weight with the ratio of corresponding densities \citep[Section 3.4.4]{liu_theory_2004}, in this case $\displaystyle {\normalized{\sgw}^{a^i_{t}}_{t-1} / {\stopgrad{\normalized{\sgw}^{a^i_t}_{t-1}}}}$. With DPF as defined in Algorithm \ref{alg:pf}, we have
\begin{align}
    \evaluate{\grad_{\theta} \log \hat{Z}_{\text{DPF}}} =
    \sum_{i=1}^N \normalized{w}_T^i \grad_{\theta} \log p_{\theta}(\lineage{x}_{1:T}^i, y_{1:T}) , \label{eq:dpf-score}
\end{align}
recovering the estimator from Eq. \ref{eq:poyadjis}. Detailed calculations are presented in the supplementary materials. This is known to be a consistent estimator of the score function \citep{poyiadjis_particle_2011} and we obtain it by applying AD to a log-likelihood estimator produced by a particle filter with minor correction to its weights. Those corrections are easy to implement, their computational overhead is small, and they do not affect the estimators produced on the forward pass.

In contrast, applying AD to a standard PF implementation, using DiCE to account for the dependence of resampling distribution on $\theta$, produces the following gradient estimators \citep[Eq. 31]{le_auto-encoding_2018}
\begin{align}
    &\evaluate{\grad_{\theta} \left( \prod_{t=1}^T \prod_{i=1}^N \frac{w_t^{a_t^i}}{\stopgrad{w_t^{a_t^i}}} \log \hat{Z}_{\text{PF}} \right)} = \\
    &\log \hat{Z}_{\text{PF}} \sum_{i=1}^N \sum_{t=1}^T \grad_{\theta} \log \normalized{w}_t^i +
            \grad_{\theta} \log \hat{Z}_{\text{PF}} = \\
    &  \sum_{i=1}^N \sum_{t=1}^T \log \hat{Z} \ \grad_{\theta} \log \normalized{w}_t^i +
         \normalized{w}_t^i \grad_{\theta} \log p_{\theta}(x_t^i, y_t | x_{t-1}^{a_t^i})    . \label{eq:aesmc-grad}
\end{align} 
The first term is typically dropped due to high-variance, introducing bias which was shown not to disappear in the infinite particle limit \citep{corenflos_differentiable_2021}. Comparing with Eq. \ref{eq:dpf-score}, this estimator does not track dependencies across time beyond a single step.

\subsection{Estimating expectations} \label{sec:expect}

While our primary goal is to obtain score function estimates by applying AD to log-likelihood estimates, particle filters can also be used to estimate expectations under filtering or smoothing distributions and we can apply AD to estimate gradients of such expectations as well. In this section we show that with the same DPF implementation, shown in Algorithm \ref{alg:pf}, AD produces sensible and easy to interpret gradient estimators with desirable theoretical properties. In the main text we only state the results, providing derivations and proofs in the supplementary materials.

Consider a function $f_{\theta} : \mathcal{X}^T \rightarrow \mathcal{R}$, differentiable in $\theta$ for all $x_{1:T}$. A particle filter yields a consistent estimator for its expectation under the posterior
\begin{align}
    \lim_{N \rightarrow \infty} \sum_{i=1}^N \normalized{w}_T^i f_{\theta}(\lineage{x}_{1:T}^i) = \expect{f_{\theta}(x_{1:T})}{x \sim p_{\theta}(x_{1:T} | y_{1:T})} .
\end{align}

Applying AD to DPF produces the following consistent estimator for the gradient of this expectation
\begin{align}
    & \grad_{\theta} \expect{f_{\theta}(x_{1:T})}{x \sim p_{\theta}(x_{1:T} | y_{1:T})} \approx
    \evaluate{\grad_{\theta} \sum_{i=1}^N \normalized{v}_T^i f_{\theta}(\lineage{x}_{1:T}^i)} = \nonumber\\
    & \sum_{i=1}^N \normalized{w}_T^i \Big( \grad_{\theta} f_{\theta}(\lineage{x}_{1:T}^i) +\nonumber\\
    &\quad \left( f_{\theta}(\lineage{x}_{1:T}^i) - \normalized{f}_{\theta} \right) \grad_{\theta} \log p_{\theta}(\lineage{x}_{1:T}^i, y_{1:T}) \Big) \label{eq:expect-dpf} ,
\end{align}
where $\normalized{f_{\theta}} = \sum_{i=1}^N \normalized{w}_T^i f_{\theta}(\lineage{x}_{1:T}^i)$.
The first term is the average value of the partial derivative of $f_{\theta}$ at the locations specified by surviving particles, while the second one can be interpreted as a DiCE-like log-derivative correction under the posterior distribution of the full ancestral lineage, with a variance-reducing baseline provided by the average value of $f_{\theta}$ obtained on the forward pass.

In contrast, applying AD to the standard PF implementation with DiCE yields
\begin{align}
     &\normalized{f}_{\theta} \sum_{t=1}^T \sum_{i=1}^N \grad_{\theta} \log \normalized{w}_{t-1}^{a_t^i} + \sum_{i=1}^N \normalized{w}_T^i \Bigg( \grad_{\theta} f_{\theta}(\lineage{x}_{1:T}^i) \ +\nonumber\\
    &\left( f_{\theta}(\lineage{x}_{1:T}^i) - \normalized{f}_{\theta} \right) \grad_{\theta} \log p_{\theta}(x_T, y_T | x_{T-1}^{a_T^i}) \Bigg) \label{eq:expect-pf} .
\end{align}
Once again, this estimators suffers from a credit assignment problem. Gradients of resampling probabilities for particles from time steps other than the last one are all given equal weight, regardless of how the particle performed subsequently.

All the gradient estimators considered so far were biased, which is to be expected given that they were obtained by differentiating biased estimators from the forward pass. However, particle filters are known to produce unbiased estimators for expectations under the \emph{unnormalized} posterior. Differentiating through DPF, we obtain the following gradient estimator
\begin{align}
    & \grad_{\theta} \expect{Z f_{\theta}(x_{1:T})}{x_{1:T} \sim p_{\theta}(x_{1:T} | y_{1:T})} \approx
    \evaluate{\grad_{\theta} \hat{Z}_{\text{DPF}} \sum_{i=1}^N \normalized{v}_T^i f_{\theta}(\lineage{x}_{1:T}^i)} = \\
    & \hat{Z} \sum_{i=1}^N \normalized{w}_T^i \left( f_{\theta}(\lineage{x}_{1:T}^i) \grad_{\theta} \log p_{\theta}(\lineage{x}_{1:T}^i, y_{1:T}) + \grad_{\theta} f_{\theta}(\lineage{x}_{1:T}^i) \right) . \label{eq:dpf-unbiased}
\end{align}
This very similar to Eq. \ref{eq:expect-dpf}, only without the baseline $\normalized{f}_{\theta}$ and it is multiplied by the marginal likelihood estimate $\hat{Z}$. It is both consistent and unbiased, as proven in the supplementary materials.

We compare it with the estimator obtained with PF and DiCE, which is also consistent and unbiased, but suffers from credit assignment problems
\begin{align}
    & \grad_{\theta} \expect{Z f_{\theta}(x_{1:T})}{x_{1:T} \sim p_{\theta}(x_{1:T} | y_{1:T})} \approx \nonumber \\
    & \hat{Z} \sum_{i=1}^N \normalized{w}_T^i \left( \grad_{\theta} f_{\theta}(\lineage{x}_{1:T}^i) +
            \left( f_{\theta}(\lineage{x}_{1:T}^i) - \normalized{f}_{\theta} \right) \grad_{\theta} \log w_T^i \right) + \nonumber \\
    & \quad\quad \hat{Z} \normalized{f}_{\theta} \sum_{t=1}^T \sum_{i=1}^N \grad_{\theta} \log \normalized{w}_{t-1}^{a_t^i} .
\end{align}

\section{Extensions} \label{sec:extensions}


\subsection{Handling proposal distributions} \label{sec:proposal}

So far we have only considered gradients with respect to model parameters $\theta$. However, it is often desirable to jointly learn separate parameters $\phi$ of the proposal distributions. Doing this requires reparameterizing $q_{\phi}$, which we assume to be given by $x_t^i = h_{\phi}(\epsilon_t^i, x_{t-1}^{a_t^i})$ with $\epsilon_t^i \sim \mathcal{N}(0,1)$. This induces a reparameterization across the entire ancestral lineage $\lineage{x}_{1:T}^i = \lineage{h}_{\phi}(\lineage{\epsilon}_{1:T}^i)$, which provides convenient notation. We then compute the result of applying automatic differentiation to $\log \hat{Z}_{\text{DPF}}$
\begin{align}
    \evaluate{\grad_{\phi} \log \hat{Z}_{\text{DPF}}} =
    \sum_{i=1}^N \normalized{w}_T^i \nabla_{\phi} \log \frac{p_{\theta}(\lineage{h}_{\phi}(\lineage{\epsilon}_{1:T}^i), y_{1:T})}{q_{\phi}(\lineage{h}_{\phi}(\lineage{\epsilon}_{1:T}^i))} .
\end{align}
This is very similar to the gradients $\grad_{\phi} \mathcal{L}_{\text{IWAE}}$ \citep[Eq. 14]{burda_importance_2016}, the only difference being that the trajectories are sampled from the particle filter rather then directly from the proposal. We can therefore reasonably expect those gradients to provide learning signal for $\phi$, although it is unclear whether it would be better than $\evaluate{\grad_{\phi} \log \hat{Z}_{\text{PF}}}$ \citep{le_auto-encoding_2018, maddison_filtering_2017, naesseth_variational_2018} due to particle degeneracy problems. Our initial experiments in this regard, presented in Section \ref{sec:agents}, were inconclusive, so we leave this question for future work.

In some settings, particularly when using the bootstrap particle filter, the proposal is derived from model parameters. For the reasons discussed in Section \ref{sec:score}, in such settings we recommend applying stop-gradient to the proposal throughout the program. This was done by \citet[Section 5.1]{le_auto-encoding_2018} without theoretical justification and we do it in our experiments as well.


\subsection{Reducing variance with additional computation} \label{sec:variance}

\begin{algorithm*}
\caption{The marginal particle filter in the usual formulation (left) and in our differentiable formulation (right).
When estimating the score function, DPF2 reduces variance from quadratic to linear in $T$. The computational cost increases from linear to quadratic in $N$.}
\label{alg:quadratic}
\begin{minipage}{0.49\textwidth}
\begin{algorithmic}
\Function{PF2}{$p$, $q$, $N$, $T$}
  \State $x_0^i \sim p(x_0)$
  \State $w_0^i = \normalized{w}_0^i = \frac{1}{N}$
  \For{$t \in 1:T$}
    \State $x^i_t \sim \sum_{j=1}^N \normalized{\sgw}_{t-1}^j q_{\phi}(x_t | x^{a_t^i}_{t-1}, y_t)$
    \State $w_{t}^n = \frac{1}{N} \frac{\sum_{i=1}^N \normalized{w}_{t-1}^i \ p_{\theta}(x_{t}^n, y_{t} | x_{t-1}^i)}{\sum_{i=1}^N \normalized{w}_{t-1}^i \ q_{\phi}(x_{t}^n | x_{t-1}^i, y_{t})}$
    \State $W_t = \sum_{i=1}^N w^i_{t}$
    \State $\normalized{w}^i_t = w^i_t / W_t$
  \EndFor
  \State $\hat{Z}_{\text{PF2}} = \prod_{t=1}^T W_t$
\EndFunction
\end{algorithmic}
\end{minipage}
\begin{minipage}{0.49\textwidth}
\begin{algorithmic}
\Function{DPF2}{$p$, $q$, $N$, $T$}
  \State $x_0^i \sim p(x_0)$
  \State $\sgw_0^i = \normalized{\sgw}_0^i = \frac{1}{N}$
  \For{$t \in 1:T$}
    \State $x^i_t \sim \sum_{j=1}^N \stopgrad{\normalized{\sgw}_{t-1}^j} q_{\phi}(x_t | x^{a_t^i}_{t-1}, y_t)$
    \State $\sgw_{t}^n = \frac{1}{N}$ \colorbox{yellow}{$\frac{\sum_{i=1}^N \normalized{\sgw}_{t-1}^i \ p_{\theta}(x_{t}^n, y_{t} | x_{t-1}^i)}{\sum_{i=1}^N \stopgrad{\normalized{\sgw}_{t-1}^i} \ q_{\phi}(x_{t}^n | x_{t-1}^i, y_{t})}$}
    \State $\Sgw_t = \sum_{i=1}^N \sgw^i_{t}$
    \State $\normalized{\sgw}^i_t = \sgw^i_t / \Sgw_t$
  \EndFor
  \State $\hat{Z}_{\text{DPF2}} = \prod_{t=1}^T \Sgw_t$
\EndFunction
\end{algorithmic}
\end{minipage}
\end{algorithm*}

\cite{poyiadjis_particle_2011} have shown that the variance of score function estimators from Eq. \ref{eq:poyadjis} grows quadratically with sequence length, but can be reduced to linear using what is now known as the marginal particle filter \citep{klaas_toward_2012}. We obtain this reduced-variance estimator with AD by stopping the gradients of the mixture component weights in the proposal distribution. This modification is shown in Algorithm \ref{alg:quadratic}.

\citet[Eq. 20]{poyiadjis_particle_2011} define their reduced-variance score function estimator as
\begin{align}
    \grad_{\theta} \log p_{\theta}(y_{1:T}) \approx \sum_{i=1}^N \normalized{w}_T^i \bar{\alpha}_T^i , \label{eq:poyiadjis-quadratic}
\end{align}
using the marginal particle filter and the following recursive definition of $\bar{\alpha}$
\begin{align}
    &\bar{\alpha}_{t+1}^i = \frac{1}{\sum_{j=1}^N \normalized{w}_{t+1}^j p_{\theta}(x_{t+1}^i | x_t^j)} \\
    &\quad\sum_{j=1}^N \normalized{w}_{t+1}^j p_{\theta}(x_{t+1}^i | x_t^j) \left( 
        \bar{\alpha}_t^j + \grad_{\theta} \log p_{\theta}(x_{t+1}^i, y_{t+1} | x_{t}^j) \right) . \nonumber \label{eq:alpha-quadratic}
\end{align}
In the supplementary materials we show that $\evaluate{\grad_{\theta} \log \hat{Z}_{\text{DPF2}}}$ produces the same estimator.

The marginal particle filter avoids issues with particle degeneracy, at the expense of not producing joint samples $\lineage{x}_{1:T}^i$. This means that it is not suitable for estimating expectations of arbitrary functions under the smoothing density, but on the flipside it achieves linear scaling of variance with $T$ for estimating the score function. Its computational cost is quadratic in the number of particles, since the mixture densities need to be evaluated for each particle separately. As shown by \citet{poyiadjis_particle_2011}, this trade-off is always favorable for sufficiently large $T$.

\citet{lai_variational_2021} propose to use the marginal particle filter to jointly learn proposal parameters $\phi$. Regardless of whether stop-gradients are applied, the key consideration is how to reparameterize sampling from the mixture proposal. \citet{lai_variational_2021} consider two methods. The first draws the mixture component first from a discrete distribution, which introduces bias in the gradients. This bias can be avoided with the stop-gradient corrections of DPF2, but that potentially introduces problems with particle degeneracy, much like in the normal particle filter. The second method is to reparameterize the mixture distribution using cumulative distribution functions \citep{graves_stochastic_2016}. In this method there is still a choice of whether to stop gradients and reparameterize a mixture with fixed weights and it is not clear whether that would perform better or worse. 

\subsection{Higher-order derivatives}

While in most applications we are interested in first-order derivatives, we can apply automatic differentiation repeatedly to obtain derivatives of higher orders. \citet{poyiadjis_particle_2011} show that the Hessian of the log-marginal likelihood can be approximated with a particle filter using Louis' identity
\begin{align}
    &\grad_{\theta}^2 \log p_{\theta}(y_{1:T}) \approx \\
    & - \sum_{i=1}^N \sum_{j=1}^N \normalized{w}_T^i \normalized{w}_T^j \grad_{\theta} \log p_{\theta}(\lineage{x}_{1:T}^i, y_{1:T}) \grad_{\theta} \log p_{\theta}(\lineage{x}_{1:T}^i, y_{1:T})^T \nonumber \\
    & + \sum_{i=1}^N \normalized{w}_T^i \Bigg(\grad_{\theta}^2 \log p_{\theta}(\lineage{x}_{1:T}^i, y_{1:T}) + \nonumber \\
    &   \quad\quad \grad_{\theta} \log p_{\theta}(\lineage{x}_{1:T}^i,, y_{1:T}) \grad_{\theta} \log p_{\theta}(\lineage{x}_{1:T}^i,, y_{1:T})^T \Bigg) \nonumber \label{eq:poyiadjis-hessian},
\end{align}
where superscript $T$ denotes transpose. This is exactly the estimator obtained by applying AD twice to DPF, specifically from $\evaluate{\grad_{\theta}^2 \log \hat{Z}_{\text{DPF}}}$, as we show in the supplementary materials. Furthermore, Hessian-vector products can be computed with reduced memory requirements in the usual fashion.

We can similarly derive higher-order derivatives for expectations under the posterior. We conjecture that DPF with automatic differentiation provides unbiased estimators for derivatives of any order of expectations under the unnormalized posterior, similarly to how DiCE produces unbiased estimators for derivatives of any order, leaving the proof for future work.

\section{Experiments}

In our experiments we focus on the task of model learning using the score function estimators given by $\evaluate{\grad_{\theta} \log \hat{Z}_{\text{DPF}}}$. We show that DPF outperforms PF on model learning, while the results for jointly learning the model and the proposal are open for interpretation.

We use the following baselines: sequential importance sampling without resampling (SIS), biased particle filter ignoring the gradients resulting from resampling (PF), and a particle filter with DiCE for ancestral variables (PF-SF). We refer to our method as DPF-SGR, for ``stop-gradient resampling'', to distinguish it from other versions of differentiable particle filterning proposed in the literature. We borrow the experimental setup from \citet{le_auto-encoding_2018}, using the code provided by those authors with our modifications needed to implement DPF-SGR. An additional experiment on a stochastic volatility model, including comparisons with the method proposed by \citet{corenflos_differentiable_2021}, is provided in the supplementary materials.

\subsection{Linear Gaussian State Space Model} \label{sec:lgssm}
In the first experiment we use a Linear Gaussian State Space Model (LGSSM) with a bootstrap proposal exactly as specified by \citet{le_auto-encoding_2018}, with two global parameters, a one dimensional latent variable and 200 time steps. Figure \ref{fig:experiments/lgssm} in our paper corresponds exactly to Figure 2 from \citet{le_auto-encoding_2018}, with DPF-SGR and PF-SF added. We find that DPF-SGR offers a noticeable improvement over the standard PF in the low particle regime, while in the high particle regime DPF-SGR achieves significantly faster convergence at the expense of a minimal decrease in the final performance.

\begin{figure}[t]
    \centering
    \includegraphics[width=0.49\textwidth]{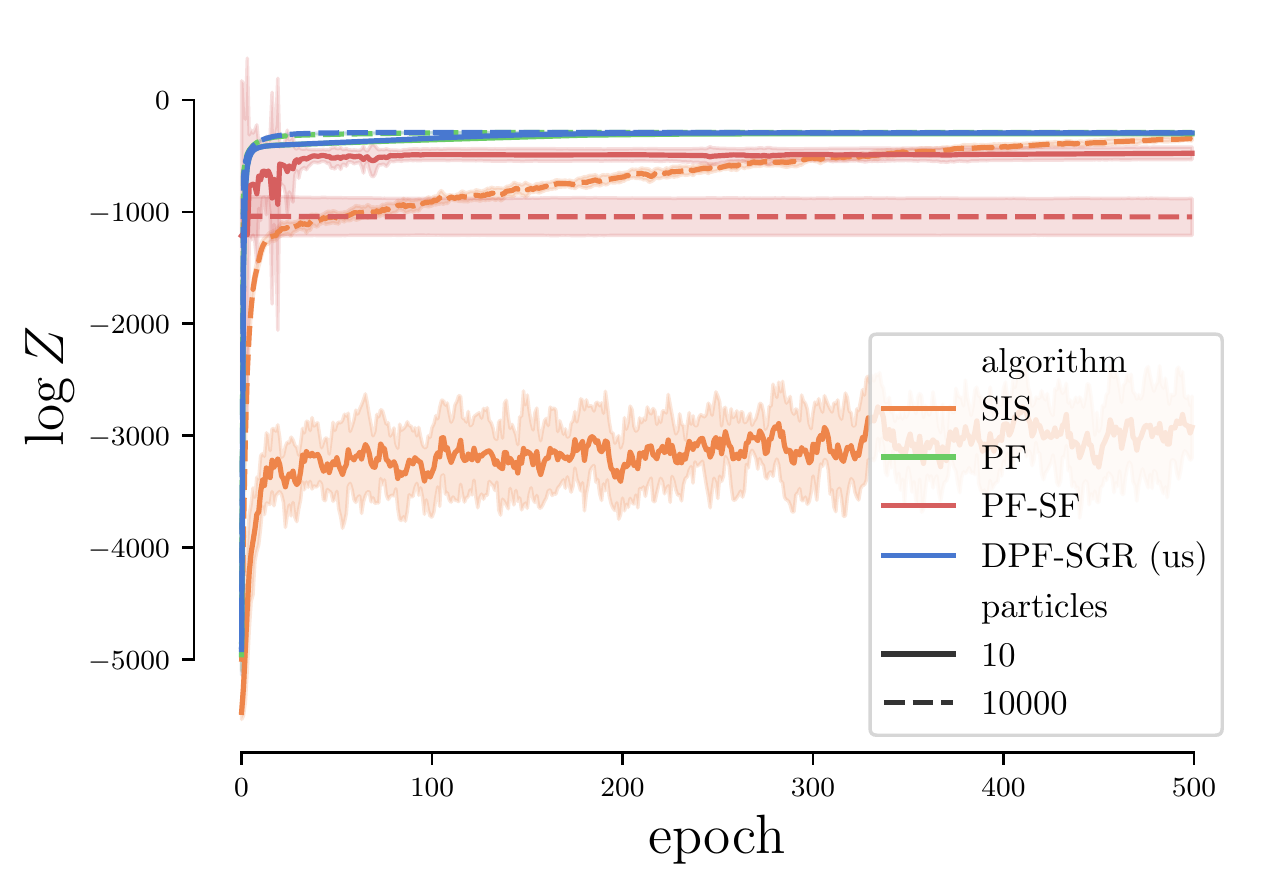}
    \includegraphics[width=0.49\textwidth]{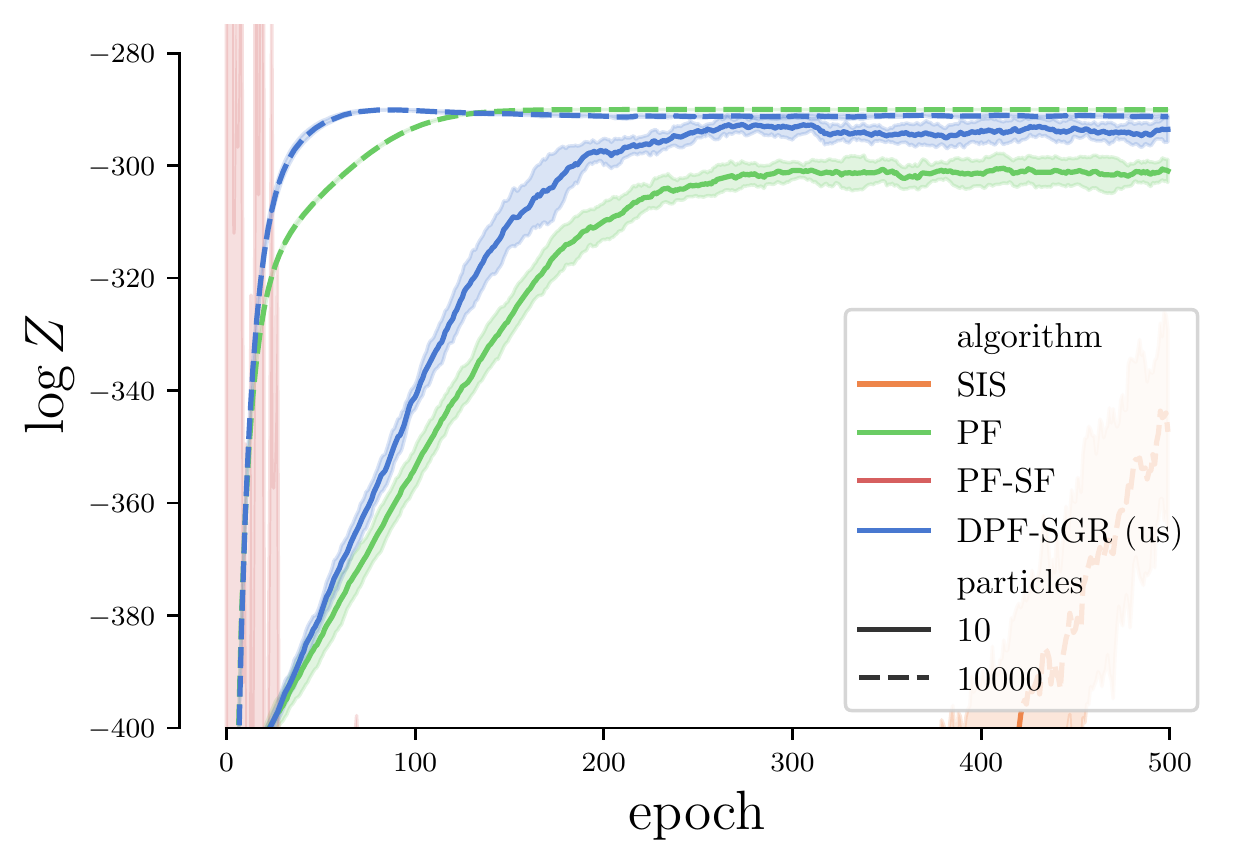}
    \caption{Estimates for the test set log-marginal likelihood throughout the training of the LGSSM model. The bottom plot is the zoomed in version of the top one, showing improvement of DPF-SGR over PF.}
    \label{fig:experiments/lgssm}
\end{figure}

\subsection{Variational Recurrent Neural Network} \label{sec:agents}

To jointly learn the model and the proposal, we train a variational recurrent neural network (VRNN) modeling the movement of agents in a 2D world where the obvservations are provided in the form of partially occluded images. We use the setting introduced by \citet{le_auto-encoding_2018}, with 32x32 images and 40 time steps. Our Figure \ref{fig:experiments/bouncing_ball} corresponds exactly to Figure 4 of \citet{le_auto-encoding_2018}, except that we show performance both on the training and test sets. While PF achieves better performance on the training set, this is largely a result of overfitting, with DPF-SGR learning slower but ultimately achieving better test set performance. It is unclear to us what general conclusions can be drawn from this experiment regarding the suitability of DPF-SGR for learning the proposal, other than that it warrants further investigation.

\begin{figure}[t]
    \centering
    \includegraphics[width=0.49\textwidth]{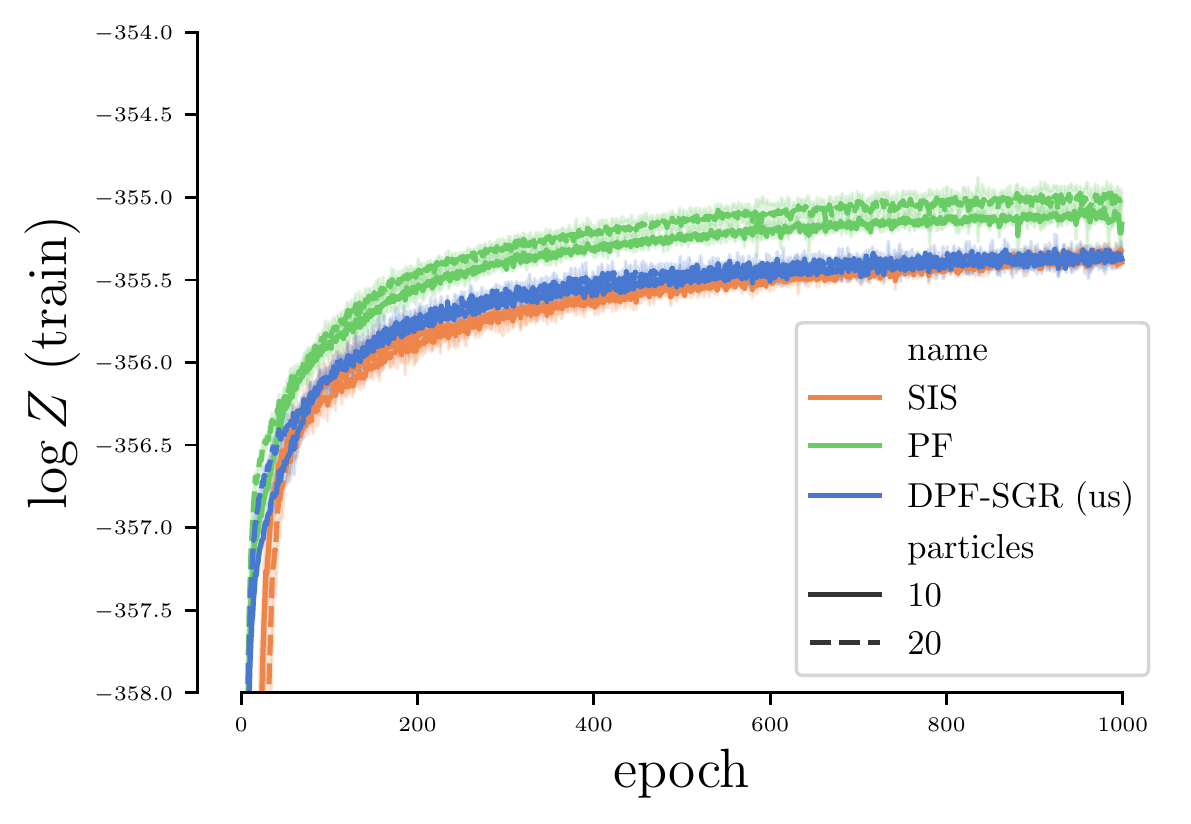}
    \includegraphics[width=0.49\textwidth]{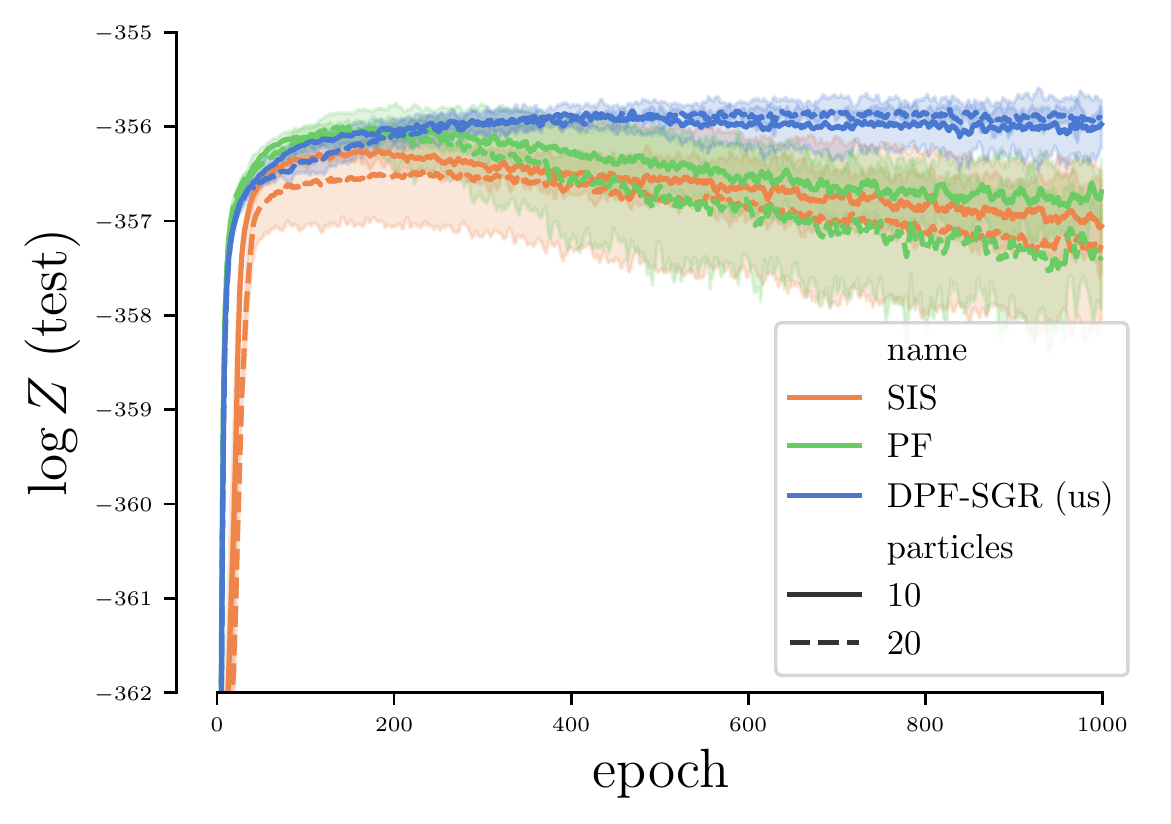}
    \caption{Log-marginal likelihood estimates for training and test set throughout the training of the VRNN model. Unlike DPF-SGR, SIS and PF show clear signs of overfitting.}
    \label{fig:experiments/bouncing_ball}
\end{figure}

\section{Related Work}

\label{sec:diff-resampling}

Due to high variance of estimators obtained by applying DiCE to the ancestral variables, the gradients associated with sampling those variables are typically ignored \citep{le_auto-encoding_2018, maddison_filtering_2017, naesseth_variational_2018}. In particular, \citet{jonschkowski_differentiable_2018} construct an algorithm they call Differentiable Particle Filtering, which explicitly ignores the propagation of gradients through the resampling steps. They suggest adding auxiliary gradients to weights in an attempt to mitigate the resulting bias as a direction of future work, which is exactly what we provide here.

There are several existing approaches to making resampling differentiable, all of which modify its behavior on the forward pass. \citet{karkus_particle_2018} resample with probabilities proportional to $\alpha w^i + \frac{1 - \alpha}{N}$, which effectively interpolates between no resampling and no gradient bias forcing the users to make an undesirable trade-off. We get the best of both worlds, retaining full advantages of resampling while allowing the gradients to propage. \citet{corenflos_differentiable_2021} replace traditional resampling with a continuous transformation of particles based on optimal transport equations. This is computationally intensive, biased, introduces additional hyperparameters, and can lead to undesirable behavior, in the extreme cases even putting the resampled particles outside of the support of the prior. On the other hand, it avoids problems with particle degeneracy while achieving full reparameterization. \citet{zhu_towards_2020} modify particles in a similar fashion using learned resampling networks, with similar disadvantages, noting that in practice propagating gradients through them leads to instabilities so they avoid it.

\section{Discussion}

We have demonstrated that applying AD to DPF recovers known estimators for the score function, which enables maximum likelihood model learning by gradient ascent. We note that those estimators target the gradients of log-marginal likelihood directly, rather than the gradients of some lower bound to it. For this reason it's unclear whether such gradients are suitable for learning the proposal, since the marginal likelihood does not depend on the proposal. The topic of constructing objectives for proposal learning is somewhat controversial, and it has in particular been shown that tightening lower bounds to log-marginal likelihood eventually impedes proposal learning \citep{rainforth_tighter_2019}. Generally, it is neither necessary nor sufficient to optimize a lower bound to log-marginal likelihood to provide good learning signal to the proposal. For example, \citet{le2018revisiting} suggest using a separate objective for proposal learning with alternating optimization, which is competitive but not clearly a superior approach. Since our score function estimators contain bias dependent on the proposal, they might be able to provide a learning signal for it, as indicated in Sections \ref{sec:proposal} and \ref{sec:agents}. We believe the question warrants further investigation.

The stop-gradient operator is relatively rarely used in theoretical formulas and many authors prefer to avoid it. We believe it presents an opportunity to write equations more closely resembling the implementation, reducing the scope for translation errors. Ultimately, stop-gradient is just a different notation for familiar concepts, in particular $\grad_{\theta} \frac{p_{\theta}(x)}{\stopgrad{p_{\theta}(x)}}$ expresses the same thing as $\grad_{\theta'} \frac{p_{\theta'}(x)}{p_{\theta}(x)} \big\rvert_{\theta' = \theta}$, only making it easier to build compound expressions. We believe that using this operator more broadly can drastically simplify certain derivations and proofs.

\section*{Acknowledgements}

We thank Tuan Anh Le, Maximilian Igl, Adrien Corenflos, and James Thornton for sharing the code for their experimental setups which we used to conduct the experiments described in this paper, and Vaden Masrani for helping to adapt this code to our needs. We also thank Marcin Tomczak, Wilder Lavington and the anonymous reviewers for providing feedback on a draft of this paper.

We acknowledge the support of the Natural Sciences and Engineering Research Council of Canada (NSERC), the Canada CIFAR AI Chairs Program, and the Intel Parallel Computing Centers program. This material is based upon work supported by the United States Air Force Research Laboratory (AFRL) under the Defense Advanced Research Projects Agency (DARPA) Data Driven Discovery Models (D3M) program (Contract No. FA8750-19-2-0222) and Learning with Less Labels (LwLL) program (Contract No.FA8750-19-C-0515). Additional support was provided by UBC's Composites Research Network (CRN), Data Science Institute (DSI) and Support for Teams to Advance Interdisciplinary Research (STAIR) Grants. This research was enabled in part by technical support and computational resources provided by WestGrid (https://www.westgrid.ca/) and Compute Canada (www.computecanada.ca).

\balance

\bibliographystyle{plainnat}
\bibliography{iai-refs,extra-refs}

\clearpage

\onecolumn
\aistatstitle{
Supplementary Materials
}

\section{Stop-gradient tutorial}

The key equations for manipulating expressions involving stop-gradient are
\begin{align}
    \evaluate{f(E_1, \dots, E_n)} &= f(\evaluate{E_1}, \dots, \evaluate{E_2}) \\
    \evaluate{\stopgrad{E}} &= \evaluate{E} \\
    \grad{\stopgrad{E}} &= 0 \\
    \evaluate{\grad{E}} &= \grad{E} \quad \text{if} \ \bot \notin E  .
\end{align}
Crucially, it is generally the case that $\evaluate{\grad E} \neq \grad \evaluate{E}$, that is different expressions involving stop-gradient can always evaluate to the same value on the forward pass, but produce different gradients. For example,
\begin{align}
    &\grad_x \evaluate{x + \stopgrad{x}^2 + x \stopgrad{x}^2} =
    \grad_x \left( \evaluate{x} + \evaluate{\stopgrad{x}^2} + \evaluate{x \stopgrad{x}^2} \right) =
    \grad_x \left( x + \evaluate{\stopgrad{x}}^2 + \evaluate{x} \evaluate{\stopgrad{x}^2} \right) = \\
    &\grad_x \left( x + x^2 + x x^2 \right) =
    \grad_x \left (x + x^2 + x^3 \right) =
    1 + 2 x + 3 x^2.
\end{align}
On the other hand,
\begin{align}
    &\evaluate{\grad_x \left(x + \stopgrad{x}^2 + x \stopgrad{x}^2 \right)} =
    \evaluate{1 + 2 \stopgrad{x} \grad_x \stopgrad{x} + \left( \grad x \right) \stopgrad{x}^2 + x \left( \grad_x \stopgrad{x}^2 \right)} = \\
    &\evaluate{1 + 2 \stopgrad{x} 0 + \stopgrad{x}^2 + x 0} =
    \evaluate{1 + \stopgrad{x}^2} =
    \evaluate{1} + \evaluate{\stopgrad{x}^2} =
    1 + \evaluate{\stopgrad{x}}^2 =
    1 + x^2 .
\end{align}
The chain rule for standard functions works normally, as does repeated differentiation. Moreover, it is possible to interleave differentiation with evaluation. For example,
\begin{align}
    &\grad_x \evaluate{\grad_x \exp \left( \stopgrad{x} + \log x \right)} =
    \grad_x \evaluate{\exp \left( \stopgrad{x} + \log x \right) \grad_x\left( \stopgrad{x} + \log x \right)} = \\
    &\grad_x \evaluate{\exp \left( \stopgrad{x} + \log x \right) \left(\grad_x \stopgrad{x} + \grad_x \log x \right)} =
    \grad_x \evaluate{\exp \left( \stopgrad{x} + \log x \right) \left(0 + \frac{1}{x} \right)} = \\
    &\grad_x \evaluate{\exp \left( \stopgrad{x} + \log x \right) \frac{1}{x}} =
    \grad_x \left( \evaluate{\exp \left( \stopgrad{x} + \log x \right)} \evaluate{\frac{1}{x}} \right) = \\
    &\grad_x \left( \exp \evaluate{\left( \stopgrad{x} + \log x \right)} \frac{1}{x} \right) =
    \grad_x \left( \exp \left( x + \log x \right) \frac{1}{x} \right) = \\
    &\frac{1}{x} \grad_x \exp \left( x + \log x \right) + \exp \left( x + \log x \right) \grad_x \frac{1}{x} = \\
    &\frac{1}{x} \exp \left( x + \log x \right) \left( 1 + \frac{1}{x} \right) - \frac{1}{x^2} \exp \left( x + \log x \right) =
    \frac{1}{x} \exp \left( x + \log x \right) .
\end{align}
All the usual properties of standard functions hold, for example both addition and multiplication are associative and commutative, and the logarithm of a product is a sum of logarithms. We do not consider limits or integrals of expressions involving stop-gradient, leaving them undefined.

The key identities showing up repeatedly in this paper are
\begin{align}
    \evaluate{\grad \frac{E}{\stopgrad{E}}} = \evaluate{\frac{\grad E}{\stopgrad{E}}} =
    \frac{\grad E}{E} = \log E ,
\end{align}
and
\begin{align}
    \evaluate{\grad \frac{\stopgrad{E}}{E}} = \evaluate{- \frac{\stopgrad{E}}{E^2} \grad E} =
    - \evaluate{\frac{\stopgrad{E}}{E^2}} \evaluate{\grad E} =
    - \frac{E}{E^2} \grad E = - \frac{\grad E}{E} = - \log E ,
\end{align}
both of which hold assuming that $E$ does not contain any stop-gradient terms.

Equipped with those, we can see what DiCE evaluates to
\begin{align}
    &\evaluate{\grad_{\theta} \left( \frac{p_{\theta}(x)}{\stopgrad{p_{\theta}(x)}} f_{\theta}(x) \right)} =
    \evaluate{\grad_{\theta} \left( \frac{p_{\theta}(x)}{\stopgrad{p_{\theta}(x)}} \right) f_{\theta}(x) +
            \frac{p_{\theta}(x)}{\stopgrad{p_{\theta}(x)}} \grad_{\theta} f_{\theta}(x)} = \\
    &f_{\theta}(x) \evaluate{\frac{\grad_{\theta} p_{\theta}(x)}{\stopgrad{p_{\theta}(x)}}} + \evaluate{\frac{p_{\theta}(x)}{\stopgrad{p_{\theta}(x)}}} \grad_{\theta} f_{\theta}(x) =
    f_{\theta}(x) \grad_{\theta} \log p_{\theta}(x) + \grad_{\theta} f_{\theta}(x) .
\end{align}

\section{Detailed derivations}

\subsection{Fisher's identity with stop-gradient}

The IWAE \citep{burda_importance_2016} objective is
\begin{align}
    \mathcal{L}_{\text{IWAE}} =
    \expect{\log \hat{Z}}{x^i \sim q_{\phi}} =
    \expect{\log \frac{1}{N} \sum_{i=1}^N w^i}{x^i \sim q_{\phi}(x)} =
    \expect{\log \frac{1}{N} \sum_{i=1}^N \frac{p_{\theta}(x^i, y)}{q_{\phi}(x^i)}}{x^i \sim q_{\phi}(x)} .
\end{align}
Automatic differentiation then computes gradients with respect to the model parameters $\theta$ as follows
\begin{align}
    \grad_{\theta} \mathcal{L}_{\text{IWAE}} &=
    \expect{\grad_{\theta} \log \frac{1}{N} \sum_{i=1}^N w^i}{x^i \sim q_{\phi}(x)} =
    \expect{\frac{\grad_{\theta} \sum_{i=1}^N w^i}{\sum_{i=1}^N w^i}}{x^i \sim q_{\phi}(x)} \\
    &= \expect{\frac{\sum_{i=1}^N w^i \grad_{\theta} \log w^i}{\sum_{i=1}^N w^i}}{x^i \sim q_{\phi}(x)} =
    \expect{\sum_{i=1}^N \normalized{w}^i \grad_{\theta} \log w^i}{x^i \sim q_{\phi}(x)} \\
    &= \expect{\sum_{i=1}^N \normalized{w}^i \grad_{\theta} \log p_{\theta}(x^i, y)}{x^i \sim q_{\phi}(x)} .
\end{align}

This is exactly the same estimator we would obtain by using Fisher's identity with a self-normalizing importance sampler to approximate the posterior
\begin{align}
    \grad_{\theta} \log p_{\theta}(y) &= \expect{\grad_{\theta} \log p_{\theta}(x, y)}{x \sim p_{\theta}(x | y)} =
    \expect{\frac{p_{\theta}(x | y)}{q_{\phi}(x)} \grad_{\theta} \log p_{\theta}(x, y)}{x \sim q_{\phi}(x)} \\
    &= \expect{\frac{w}{p_{\theta}(y)} \grad_{\theta} \log p_{\theta}(x, y)}{x \sim q_{\phi}(x)} =
    \expect{\frac{1}{N} \sum_{i=1}^N \frac{w^i}{p_{\theta}(y)} \grad_{\theta} \log p_{\theta}(x^i, y)}{x^i \sim q_{\phi}(x)} \\
    &\approx \expect{\sum_{i=1}^N \normalized{w}^i \grad_{\theta} \log p_{\theta}(x^i, y)}{x^i \sim q_{\phi}(x)} \label{eq:fisher-snis} .
\end{align}

However, the two methods are no longer equivalent if the proposal distribution $q_{
\theta}$ also depends on model parameters, such as when using the prior $p_{\theta}(x)$. Furthermore, suppose it is not reparameterizable, so we apply DiCE to obtain the gradient estimators
\begin{align}
    & \grad_{\theta} \expect{\log \frac{1}{N} \sum_{i=1}^N \frac{p_{\theta}(x^i, y)}{q_{\theta}(x^i)}}{x^i \sim q_{\theta}(x)} =
    \expect{\evaluate{\grad_{\theta} \left(\prod_{i=1}^N \frac{p_{\theta}(x^i)}{\stopgrad{q_{\theta}(x^i)}} \log \frac{1}{N} \sum_{i=1}^N \frac{p_{\theta}(x^i, y)}{q_{\theta}(x^i)} \right)}}{x^i \sim q_{\theta}(x)} = \\
    & \expect{\evaluate{\left( \grad_{\theta} \prod_{i=1}^N \frac{p_{\theta}(x^i)}{\stopgrad{q_{\theta}(x^i)}} \right) 
                    \log \frac{1}{N} \sum_{i=1}^N \frac{p_{\theta}(x^i, y)}{q_{\theta}(x^i)} +
              \left(\prod_{i=1}^N \frac{p_{\theta}(x^i)}{\stopgrad{q_{\theta}(x^i)}}\right)
                    \sum_{i=1}^N \normalized{w}^i \grad_{\theta} \log \frac{p_{\theta}(x^i, y)}{q_{\theta}(x^i)}}
              }{x^i \sim p_{\theta}(x))} = \\
    & \expect{\log \hat{Z} \ \sum_{i=1}^N \grad_{\theta} \log p_{\theta}(x^i) +
              \sum_{i=1}^N \normalized{w}^i \grad_{\theta} \log \frac{p_{\theta}(x^i, y)}{q_{\theta}(x^i)}
              }{x^i \sim q_{\theta}(x)} .
\end{align}

\subsection{Score function estimation with DPF}

Let $\Lineage_t^i = \frac{p_{\theta}(\lineage{x}_{1:t}^i, y_{1:t})}{q_{\phi}(\lineage{x}_{1:t}^i)}$, which can be thought of as the importance weight accumulated along the full ancestral lineage of a particle $i$ at time $t$ and satisfies the recursion $\Lineage_t^i = N w_t^i \Lineage_{t-1}^{a_t^i}$

First, we derive a formula for the modified weights at given time steps
\begin{align}
    \sgw_t^i &= \frac{p_{\theta}(x_t^i, y_t | x_{t-1}^{a_t^i})}{q_{\phi}(x_t^i | x_{t-1}^{a_t^i})} \ \resampled{\sgw}_t^i
    = w_t^i \frac{\normalized{\sgw}_{t-1}^{a_t^i}}{\stopgrad{\normalized{\sgw}_{t-1}^{a_t^i}}}
    = w_t^i \frac{\sgw_{t-1}^{a_t^i}}{\stopgrad{\sgw_{t-1}^{a_t^i}}} \frac{\stopgrad{\Sgw_{t-1}}}{\Sgw_{t-1}} \label{eq:sgw-base} .
\end{align}
This allows us to derive and solve the following recursive equation
\begin{align}
    \frac{\sgw_t^i}{\stopgrad{\sgw_t^i}} = \frac{w_t^i}{\stopgrad{w_t^i}} \frac{\sgw_{t-1}^{a_t^i}}{\stopgrad{\sgw_{t-1}^{a_t^i}}} \frac{\stopgrad{\Sgw_{t-1}}}{\Sgw_{t-1}} = \frac{\Lineage_t^i}{\stopgrad{\Lineage_t^i}} \prod_{j=1}^{t-1} \frac{\stopgrad{\Sgw_j}}{\Sgw_j} .
\end{align}
Plugging this back into Eq. \ref{eq:sgw-base}, and noting that $\stopgrad{\sgw_t^i} = \stopgrad{w_t^i}$, we obtain the following formula
\begin{align}
    \sgw_t^i =
    w_t^i \frac{\Lineage_{t-1}^{a_t^i}}{\stopgrad{\Lineage_{t-1}^{a_t^i}}} \prod_{j=1}^{t-1} \frac{\stopgrad{\Sgw_j}}{\Sgw_j} =
    \stopgrad{\sgw_t^i} \frac{\Lineage_{t}^{i}}{\stopgrad{\Lineage_{t}^{i}}} \prod_{j=1}^{t-1} \frac{\stopgrad{\Sgw_j}}{\Sgw_j} \label{eq:sgw-final}
\end{align}

With that, we can derive an alternative form of $\hat{Z}_{\text{DPF}}$
\begin{align}
    \hat{Z}_{\text{DPF}} &= \prod_{t=1}^T \Sgw_t = \left( \prod_{t=1}^{T-1} \Sgw_t \right) \sum_{i=1}^N \sgw_T^i =
    \left( \prod_{t=1}^{T-1} \stopgrad{\Sgw_t} \right) \sum_{i=1}^N \stopgrad{\sgw_T^i} \frac{\Lineage_T^i}{\stopgrad{\Lineage_T^{i}}} \\
    &= \left( \prod_{t=1}^{T-1} \stopgrad{\Sgw_t} \right) \sum_{i=1}^N \stopgrad{\sgw_T^i} \frac{\Lineage_T^i}{\stopgrad{\Lineage_T^{i}}} \frac{\stopgrad{\Sgw_T}}{\sum_{i=1}^N \stopgrad{\sgw_T^i}} \\
    &= \left( \prod_{t=1}^{T} \stopgrad{\Sgw_t} \right)  \frac{\sum_{i=1}^N \stopgrad{\sgw_T^i} \frac{\Lineage_T^i}{\stopgrad{\Lineage_T^{i}}}}{\sum_{i=1}^N \stopgrad{\sgw_T^i}} =
    \stopgrad{\hat{Z}_{\text{DPF}}} \sum_{i=1}^N \stopgrad{\normalized{\sgw}_T^i} \frac{\Lineage_T^i}{\stopgrad{\Lineage_T^{i}}} \label{eq:z-dpf} .
\end{align}

This form allows us to evaluate the gradient computed by automatic differentiation
\begin{align}
    \evaluate{\nabla_{\theta} \hat{Z}_{\text{DPF}}} =
    \evaluate{\hat{Z}_{\text{DPF}}} \sum_{i=1}^N \evaluate{\normalized{\sgw}_T^i} \evaluate{\frac{\nabla_{\theta} \Lineage_T^i}{\stopgrad{\Lineage_T^{i}}}} =
    \evaluate{\hat{Z}_{\text{DPF}}} \sum_{i=1}^N \evaluate{\normalized{\sgw}_T^i} \nabla_{\theta} \log \Lineage_T^i
    \label{eq:supp-z-grad-dpf} .
\end{align}

The estimator for the score is obtained using the chain rule for the logarithm
\begin{align}
    \evaluate{\nabla_{\theta} \log \hat{Z}_{\text{DPF}}} = \frac{\evaluate{\nabla_{\theta} \hat{Z}_{\text{DPF}}}}{\evaluate{\hat{Z}_{\text{DPF}}}} =
    \sum_{i=1}^N \evaluate{\normalized{\sgw}_T^i} \nabla_{\theta} \log \Lineage_T^i =
    \sum_{i=1}^N \normalized{w}_T^i \nabla_{\theta} \log p_{\theta}(\lineage{x}_{1:T}^i, y_{1:T}) . \label{eq:supp-dpf-score}
\end{align}

\subsection{Estimators of expectations}

Consider a function $f_{\theta} : \mathcal{X}^T \rightarrow \mathcal{R}$, differentiable in $\theta$ for all $x_{1:T}$. A particle filter yields a consistent estimator for its expectation under the posterior
\begin{align}
    \lim_{N \rightarrow \infty} \sum_{i=1}^N \normalized{w}_T^i f_{\theta}(\lineage{x}_{1:T}^i) = \expect{f_{\theta}(x_{1:T})}{x \sim p_{\theta}(x_{1:T} | y_{1:T})} .
\end{align}
Differentiating through the usual PF implementation, we obtain the following gradient estimator
\begin{align}
    & \grad_{\theta} \expect{f_{\theta}(x_{1:T})}{x \sim p_{\theta}(x_{1:T} | y_{1:T})} \approx
    \sum_{i=1}^N \normalized{w}_T^i \left( f_{\theta}(\lineage{x}_{1:T}^i) \grad_{\theta} \log \normalized{w}_T^i + \grad_{\theta} f_{\theta}(\lineage{x}_{1:T}^i) \right) = \\
    & - \grad_{\theta} \log W_T \sum_{i=1}^N \normalized{w}_T^i f_{\theta}(\lineage{x}_{1:T}^i) +
        \sum_{i=1}^N \normalized{w}_T^i \left( f_{\theta}(\lineage{x}_{1:T}^i) \grad_{\theta} \log w_T^i + \grad_{\theta} f_{\theta}(\lineage{x}_{1:T}^i) \right) = \\
    & - \grad_{\theta} \log W_T \normalized{f}_{\theta} +
        \sum_{i=1}^N \normalized{w}_T^i \left( f_{\theta}(\lineage{x}_{1:T}^i) \grad_{\theta} \log w_T^i + \grad_{\theta} f_{\theta}(\lineage{x}_{1:T}^i) \right) = \\
    & - \normalized{f}_{\theta} \sum_{i=1}^N \normalized{w}_T^i \grad_{\theta} \log w_T^i +
        \sum_{i=1}^N \normalized{w}_T^i \left( f_{\theta}(\lineage{x}_{1:T}^i) \grad_{\theta} \log w_T^i + \grad_{\theta} f_{\theta}(\lineage{x}_{1:T}^i) \right) = \\
    & \sum_{i=1}^N \normalized{w}_T^i \left( \left( f_{\theta}(\lineage{x}_{1:T}^i) - \normalized{f}_{\theta} \right) \grad_{\theta} \log p_{\theta}(x_T, y_T | x_{T-1}^{a_T^i}) + \grad_{\theta} f_{\theta}(\lineage{x}_{1:T}^i) \right) \label{eq:supp-expect-pf} .
\end{align}
Furthermore, employing DiCE on ancestral variables results in the following correction
\begin{align}
    & \normalized{f}_{\theta} \sum_{t=1}^T \sum_{i=1}^N \grad_{\theta} \log \normalized{w}_{t-1}^{a_t^i} =
    \normalized{f}_{\theta} \sum_{t=1}^T \sum_{i=1}^N \grad_{\theta} \left( \log w_{t-1}^{a_t^i} - \log W_{t-1} \right) = \\
    & \normalized{f}_{\theta} \sum_{t=1}^T \sum_{i=1}^N \left( \grad_{\theta} \log w_{t-1}^{a_t^i} - \sum_{j=1}^N \normalized{w}_{t-1}^i \grad_{\theta} \log w_{t-1}^i \right) .
\end{align}

When differentiating through the implementation of the same estimator obtained by DPF, we obtain the following gradient estimator, using Eq. \ref{eq:sgw-final} and \ref{eq:supp-dpf-score}
\begin{align}
    & \grad_{\theta} \expect{f_{\theta}(x_{1:T})}{x \sim p_{\theta}(x_{1:T} | y_{1:T})} \approx
    \evaluate{\grad_{\theta} \sum_{i=1}^N \normalized{v}_T^i f_{\theta}(\lineage{x}_{1:T}^i)} = \\
    & \evaluate{\grad_{\theta} \sum_{i=1}^N \stopgrad{\normalized{v}_T^i} \frac{L_T^i}{\stopgrad{L_T^i}} \frac{\stopgrad{\hat{Z}_{\text{DPF}}}}{\hat{Z}_{\text{DPF}}} f_{\theta}(\lineage{x}_{1:T}^i)} = \\
    & \sum_{i=1}^N \normalized{w}_T^i \left( f_{\theta}(\lineage{x}_{1:T}^i) \grad_{\theta} \log L_T^i + \grad_{\theta} f_{\theta}(\lineage{x}_{1:T}^i) \right) - \normalized{f}_{\theta} \evaluate{\grad_{\theta} \log \hat{Z}_{\text{DPF}}} = \\
    & \sum_{i=1}^N \normalized{w}_T^i \left( \left( f_{\theta}(\lineage{x}_{1:T}^i) - \normalized{f}_{\theta} \right) \grad_{\theta} \log p_{\theta}(\lineage{x}_{1:T}^i, y_{1:T}) + \grad_{\theta} f_{\theta}(\lineage{x}_{1:T}^i) \right) \label{eq:supp-expect-dpf} .
\end{align}

For the expectations under the unnormalized posterior, differentiating through DPF produces the following gradient estimator
\begin{align}
    & \grad_{\theta} \expect{Z f_{\theta}(x_{1:T})}{x_{1:T} \sim p_{\theta}(x_{1:T} | y_{1:T})} \approx
    \evaluate{\grad_{\theta} \hat{Z}_{\text{DPF}} \sum_{i=1}^N \normalized{v}_T^i f_{\theta}(\lineage{x}_{1:T}^i)} = \\
    & \evaluate{\grad_{\theta} \hat{Z}_{\text{DPF}} \sum_{i=1}^N \stopgrad{\normalized{v}_T^i} \frac{L_T^i}{\stopgrad{L_T^i}} \frac{\stopgrad{\hat{Z}_{\text{DPF}}}}{\hat{Z}_{\text{DPF}}} f_{\theta}(\lineage{x}_{1:T}^i)} = \\
    & \hat{Z} \sum_{i=1}^N \normalized{w}_T^i \left( f_{\theta}(\lineage{x}_{1:T}^i) \grad_{\theta} \log L_T^i + \grad_{\theta} f_{\theta}(\lineage{x}_{1:T}^i) \right) = \\
    & \hat{Z} \sum_{i=1}^N \normalized{w}_T^i \left( f_{\theta}(\lineage{x}_{1:T}^i) \grad_{\theta} \log p_{\theta}(\lineage{x}_{1:T}^i, y_{1:T}) + \grad_{\theta} f_{\theta}(\lineage{x}_{1:T}^i) \right) . \label{eq:supp-dpf-unbiased}
\end{align}

We compare it with the estimator obtained with PF and DiCE, which is also unbiased
\begin{align}
    & \grad_{\theta} \expect{Z f_{\theta}(x_{1:T})}{x_{1:T} \sim p_{\theta}(x_{1:T} | y_{1:T})} \approx \\
    & \normalized{f}_{\theta} \grad_{\theta} \hat{Z}_{\text{PF}} +
        \hat{Z}_{\text{PF}} \grad_{\theta} \sum_{i=1}^N \normalized{w}_T^i f_{\theta}(\lineage{x}_{1:T}^i) +
        \hat{Z}_{\text{PF}} \normalized{f}_{\theta} \sum_{t=1}^T \sum_{i=1}^N \grad_{\theta} \log \normalized{w}_{t-1}^{a_t^i} = \\
    & \hat{Z} \normalized{f}_{\theta} \sum_{t=1}^T \grad_{\theta} \log W_t +
        \hat{Z} \sum_{i=1}^N \normalized{w}_T^i \left( \grad_{\theta} f_{\theta}(\lineage{x}_{1:T}^i) +
            \left( f_{\theta}(\lineage{x}_{1:T}^i) - \normalized{f}_{\theta} \right) \grad_{\theta} \log w_T^i \right) + \\
    & \quad \hat{Z} \normalized{f}_{\theta} \sum_{t=1}^T \sum_{i=1}^N \left( \grad_{\theta} \log w_{t-1}^{a_t^i} - \sum_{j=1}^N \normalized{w}_{t-1}^i \grad_{\theta} \log w_{t-1}^i \right) = \\
    & \hat{Z} \normalized{f}_{\theta} \sum_{t=1}^T \grad_{\theta} \log W_t +
        \hat{Z} \sum_{i=1}^N \normalized{w}_T^i \left( \grad_{\theta} f_{\theta}(\lineage{x}_{1:T}^i) +
            \left( f_{\theta}(\lineage{x}_{1:T}^i) - \normalized{f}_{\theta} \right) \grad_{\theta} \log w_T^i \right) + \\
    & \quad \hat{Z} \normalized{f}_{\theta} \sum_{t=1}^T \sum_{i=1}^N \left( \grad_{\theta} \log w_{t-1}^{a_t^i} - \sum_{j=1}^N \normalized{w}_{t-1}^i \grad_{\theta} \log w_{t-1}^i \right) = \\
    & \hat{Z} \sum_{i=1}^N \normalized{w}_T^i \left( \grad_{\theta} f_{\theta}(\lineage{x}_{1:T}^i) +
            \left( f_{\theta}(\lineage{x}_{1:T}^i) - \normalized{f}_{\theta} \right) \grad_{\theta} \log w_T^i \right) + \\
    & \quad \hat{Z} \normalized{f}_{\theta} \sum_{t=1}^T \sum_{i=1}^N \left( \grad_{\theta} \log w_{t-1}^{a_t^i} + \normalized{w}_{t}^i \grad_{\theta} \log w_t^i - \sum_{j=1}^N \normalized{w}_{t-1}^i \grad_{\theta} \log w_{t-1}^i \right) .
\end{align}

\subsection{Gradients for proposal learning}

The derivation leading up to Eq. \ref{eq:supp-z-grad-dpf} is valid for $\grad_{\phi}$. Then we compute the gradient of $\log L_T^i$ assuming reparameterization
\begin{align}
    \evaluate{\grad_{\phi} \log \hat{Z}_{\text{DPF}}} =
    \frac{\evaluate{\nabla_{\phi} \hat{Z}_{\text{DPF}}}}{\evaluate{\hat{Z}_{\text{DPF}}}} =
    \sum_{i=1}^N \evaluate{\normalized{\sgw}_T^i} \nabla_{\phi} \log \Lineage_T^i =
    \sum_{i=1}^N \normalized{w}_T^i \nabla_{\phi} \log \frac{p_{\theta}(\lineage{h}_{\phi}(\lineage{\epsilon}_{1:T}^i), y_{1:T})}{q_{\phi}(\lineage{h}_{\phi}(\lineage{\epsilon}_{1:T}^i))} .
\end{align}

\subsection{Score estimation with marginal particle filters}

\citet[Eq. 20]{poyiadjis_particle_2011} define their score function estimator as
\begin{align}
    \grad_{\theta} \log p_{\theta}(y_{1:T}) \approx \sum_{i=1}^N \normalized{w}_T^i \bar{\alpha}_T^i , \label{eq:supp-poyiadjis-quadratic}
\end{align}
using the marginal particle filter and the following recursive definition of $\alpha$
\begin{align}
    \bar{\alpha}_{t+1}^i &= \frac{\sum_{j=1}^N \normalized{w}_{t+1}^j p_{\theta}(x_{t+1}^i | x_t^j) \left( 
        \bar{\alpha}_t^j + \grad_{\theta} \log p_{\theta}(x_{t+1}^i, y_{t+1}) \right)}
        {\sum_{j=1}^N \normalized{w}_{t+1}^j p_{\theta}(x_{t+1}^i | x_t^j)} \\
    &=\frac{\sum_{j=1}^N \normalized{w}_{t+1}^j p_{\theta}(x_{t+1}^i, y_{t+1} | x_t^j) \left( 
        \bar{\alpha}_t^j + \grad_{\theta} \log p_{\theta}(x_{t+1}^i, y_{t+1}) \right)}
        {\sum_{j=1}^N \normalized{w}_{t+1}^j p_{\theta}(x_{t+1}^i, y_{t+1} | x_t^j)} . \label{eq:supp-alpha-quadratic}
\end{align}

We show that the same result is obtained by differentiating $\log \hat{Z}_{\text{DPF2}}$.
To find the corresponding recursion, we expand the definition of $v_{t}^i$ in the following formula
\begin{align}
    & \evaluate{\grad_{\theta} \frac{v_{t}^i}{\stopgrad{v_{t}^i}} \prod_{k=1}^{t-1} \frac{V_{k}}{\stopgrad{V_{k}}}} = \\
    & \evaluate{\grad_{\theta} \frac{\sum_{i=1}^N \normalized{\sgw}_{t-1}^i \ p_{\theta}(x_{t}^n, y_{t} | x_{t-1}^i)}
                                   {\stopgrad{\sum_{i=1}^N \normalized{\sgw}_{t-1}^i \ p_{\theta}(x_{t}^n, y_{t} | x_{t-1}^i)}}
                             \frac{\stopgrad{\sum_{i=1}^N \stopgrad{\normalized{\sgw}_{t-1}^i} \ q_{\phi}(x_{t}^n | x_{t-1}^i, y_{t})}}
                                  {\sum_{i=1}^N \stopgrad{\normalized{\sgw}_{t-1}^i} \ q_{\phi}(x_{t}^n | x_{t-1}^i, y_{t})}
                             \prod_{k=1}^{t-1} \frac{V_{k}}{\stopgrad{V_{k}}}
            } = \\
    & \evaluate{\grad_{\theta} \frac{\sum_{i=1}^N \stopgrad{\normalized{\sgw}_{t-1}^i} \ p_{\theta}(x_{t}^n, y_{t} | x_{t-1}^i)
                \frac{\normalized{\sgw}_{t-1}^i}{\stopgrad{\normalized{\sgw}_{t-1}^i}} \prod_{k=1}^{t-1} \frac{V_{k}}{\stopgrad{V_{k}}}
                }
                                   {\stopgrad{\sum_{i=1}^N \normalized{\sgw}_{t-1}^i \ p_{\theta}(x_{t}^n, y_{t} | x_{t-1}^i)}}
            } = \\
    & \frac{\sum_{i=1}^N \normalized{w}_{t-1}^i \evaluate{\grad_{\theta} \left( p_{\theta}(x_{t}^n, y_{t} | x_{t-1}^i)
                \frac{\sgw_{t-1}^i}{\stopgrad{\sgw_{t-1}^i}} \prod_{k=1}^{t-2} \frac{V_{k}}{\stopgrad{V_{k}}} \right) }}
           {\sum_{i=1}^N \normalized{w}_{t-1}^i p_{\theta}(x_{t}^n, y_{t} | x_{t-1}^i)} = \\
    & \frac{\sum_{i=1}^N \normalized{w}_{t-1}^i \left(\grad_{\theta} p_{\theta}(x_{t}^n, y_{t} | x_{t-1}^i) + p_{\theta}(x_{t}^n, y_{t} | x_{t-1}^i) \evaluate{\grad_{\theta} \frac{\sgw_{t-1}^i}{\stopgrad{\sgw_{t-1}^i}} \prod_{k=1}^{t-2} \frac{V_{k}}{\stopgrad{V_{k}}}}
                 \right)}
           {\sum_{i=1}^N \normalized{w}_{t-1}^i p_{\theta}(x_{t}^n, y_{t} | x_{t-1}^i)} = \\
    & \frac{\sum_{i=1}^N \normalized{w}_{t-1}^i p_{\theta}(x_{t}^n, y_{t} | x_{t-1}^i) \left(\grad_{\theta} \log p_{\theta}(x_{t}^n, y_{t} | x_{t-1}^i) + \evaluate{\grad_{\theta} \frac{\sgw_{t-1}^i}{\stopgrad{\sgw_{t-1}^i}} \prod_{k=1}^{t-2} \frac{V_{k}}{\stopgrad{V_{k}}}}
                 \right)}
           {\sum_{i=1}^N \normalized{w}_{t-1}^i p_{\theta}(x_{t}^n, y_{t} | x_{t-1}^i)} .
\end{align}
This is exactly the same recursive relationship as defined by Eq. \ref{eq:supp-alpha-quadratic}, so we have $ \evaluate{\grad_{\theta} \frac{v_{t}^i}{\stopgrad{v_{t}^i}} \prod_{k=1}^{t-1} \frac{V_{k}}{\stopgrad{V_{k}}}} = \bar{\alpha}_t^i$.
Differentiating through DPF2 produces the following expression
\begin{align}
    & \evaluate{\grad_{\theta} \log \hat{Z}_{\text{DPF2}}} = \frac{\evaluate{\grad_{\theta} \hat{Z}_{\text{DPF2}}}}{\hat{Z}} =
    \frac{1}{\hat{Z}} \evaluate{\grad_{\theta} \prod_{t=1}^T V_t} =
    \frac{1}{\hat{Z}} \evaluate{\grad_{\theta} \sum_{i=1}^N v_T^i \prod_{t=1}^{T-1} V_t} = \\
    & \frac{1}{\hat{Z}} \evaluate{\grad_{\theta} \sum_{i=1}^N \stopgrad{\normalized{v}_T^i V_T} \frac{v_T^i}{\stopgrad{v_T^i}} \prod_{t=1}^{T-1} V_t} =
    \frac{1}{\hat{Z}} \evaluate{\grad_{\theta} \stopgrad{\hat{Z}} \sum_{i=1}^N \stopgrad{\normalized{v}_T^i} \frac{v_T^i}{\stopgrad{v_T^i}} \prod_{t=1}^{T-1} \frac{V_t}{\stopgrad{V_t}}} = \\
    & \sum_{i=1}^N \normalized{w}_T^i \evaluate{\grad_{\theta} \frac{v_T^i}{\stopgrad{v_T^i}} \prod_{t=1}^{T-1} \frac{V_t}{\stopgrad{V_t}}} =
    \sum_{i=1}^N \normalized{w}_T^i \bar{\alpha}_T^i ,
\end{align}
recovering Eq. \ref{eq:supp-poyiadjis-quadratic}.

\subsection{Hessian of log-marginal likelihood with DPF}

While in most applications we are interested in first-order derivatives, we can apply automatic differentiation repeatedly to obtain derivatives of higher orders. \citet{poyiadjis_particle_2011} show that the Hessian of the log-marginal likelihood can be approximated with a particle filter using Louis' identity
\begin{align}
    &\grad_{\theta}^2 \log p_{\theta}(y_{1:T}) =
        - \grad_{\theta} \log p_{\theta}(y_{1:T}) \grad_{\theta} \log p_{\theta}(y_{1:T})^T + \\
    &    \expect{\grad_{\theta}^2 \log p_{\theta}(x_{1:T}, y_{1:T}) +
            \grad_{\theta} \log p_{\theta}(x_{1:T}, y_{1:T}) \grad_{\theta} \log p_{\theta}(x_{1:T}, y_{1:T})^T}
            {x_{1:T} \sim p_{\theta}(x_{1:T} | y_{1:T})} \approx \\
    & - \sum_{i=1}^N \sum_{j=1}^N \normalized{w}_T^i \normalized{w}_T^j \grad_{\theta} \log p_{\theta}(\lineage{x}_{1:T}^i, y_{1:T}) \grad_{\theta} \log p_{\theta}(\lineage{x}_{1:T}^i, y_{1:T})^T + \\
    & \sum_{i=1}^N \normalized{w}_T^i \left(\grad_{\theta}^2 \log p_{\theta}(\lineage{x}_{1:T}^i, y_{1:T}) +
            \grad_{\theta} \log p_{\theta}(\lineage{x}_{1:T}^i,, y_{1:T}) \grad_{\theta} \log p_{\theta}(\lineage{x}_{1:T}^i,, y_{1:T})^T \right) \label{eq:supp-poyiadjis-hessian},
\end{align}
where superscript $T$ denotes transpose.

We can derive the estimator produced by applying automatic differentiation twice to DPF using Eq. \ref{eq:z-dpf}
\begin{align}
    & \evaluate{\grad_{\theta}^2 \log \hat{Z}_{\text{DPF}}} =
    \evaluate{\grad_{\theta} \frac{\grad_{\theta} \hat{Z}_{\text{DPF}}}{\hat{Z}_{\text{DPF}}}} =
    \evaluate{\grad_{\theta} \left( \frac{\stopgrad{\hat{Z}_{\text{DPF}}}}{\hat{Z}_{\text{DPF}}} \sum_{i=1}^N \stopgrad{\normalized{\sgw}_T^i} \frac{\grad_{\theta} \Lineage_T^i}{\stopgrad{\Lineage_T^{i}}} \right)} = \\
    & - \evaluate{\grad_{\theta} \log \hat{Z}_{\text{DPF}}} \sum_{i=1}^N \normalized{w}_T^i \grad_{\theta} \log L_T^i +
    \sum_{i=1}^N \normalized{w}_T^i \evaluate{\frac{\grad_{\theta}^2 L_T^i}{\stopgrad{L_T^i}}} = \\
    & - \sum_{i=1}^N \sum_{j=1}^N \normalized{w}_T^i \normalized{w}_T^j \grad_{\theta} \log L_T^i (\grad_{\theta} \log L_T^j)^T +
    \sum_{i=1}^N \normalized{w}_T^i \evaluate{\grad_{\theta} \left( \frac{L_T^i}{\stopgrad{L_T^i}} (\grad_{\theta} \log L_T^i)^T \right)} = \\
    & - \sum_{i=1}^N \sum_{j=1}^N \normalized{w}_T^i \normalized{w}_T^j \grad_{\theta} \log L_T^i (\grad_{\theta} \log L_T^j)^T +
    \sum_{i=1}^N \normalized{w}_T^i \left( \grad_{\theta} \log L_T^i (\grad_{\theta} \log L_T^i)^T + \grad_{\theta}^2 \log L_T^i  \right) \label{eq:dpf-hessian} ,
\end{align}
which is exactly the same as Eq. \ref{eq:supp-poyiadjis-hessian}. Furthermore, Hessian-vector products can be computed with reduced memory requirements in the usual fashion.

\section{Proofs}

Consistency of the DPF estimator for the expectation under the posterior (Eq. \ref{eq:supp-expect-dpf}) can be shown as follows
\begin{align}
    & \lim_{N \rightarrow \infty} \sum_{i=1}^N \normalized{w}_T^i \left( \left( f_{\theta}(\lineage{x}_{1:T}^i) - \normalized{f}_{\theta} \right) \grad_{\theta} \log p_{\theta}(\lineage{x}_{1:T}^i, y_{1:T}) + \grad_{\theta} f_{\theta}(\lineage{x}_{1:T}^i) \right) = \\
    & \lim_{N \rightarrow \infty} \sum_{i=1}^N \normalized{w}_T^i \left( \left( f_{\theta}(\lineage{x}_{1:T}^i) - \normalized{f}_{\theta} \right) \grad_{\theta} \log p_{\theta}(\lineage{x}_{1:T}^i | y_{1:T}) + \grad_{\theta} f_{\theta}(\lineage{x}_{1:T}^i) \right) = \\
    & \expect{\left( f_{\theta}(x_{1:T}) - \expect{f_{\theta}(x'_{1:T})}{x'_{1:T} \sim p_{\theta}(x_{1:T} | y_{1:T})} \right) \grad_{\theta} \log p_{\theta}(x_{1:T} | y_{1:T}) + \grad_{\theta} f_{\theta}(x_{1:T})}
        {x_{1:T} \sim p_{\theta}(x_{1:T} | y_{1:T})} = \\
    & \expect{f_{\theta}(x_{1:T}) \grad_{\theta} \log p_{\theta}(x_{1:T} | y_{1:T}) + \grad_{\theta} f_{\theta}(x_{1:T})}
        {x_{1:T} \sim p_{\theta}(x_{1:T} | y_{1:T})} =
    \grad_{\theta} \expect{f_{\theta}(x_{1:T})}{x_{1:T} \sim p_{\theta}(x_{1:T} | y_{1:T})} .
\end{align}

Similarly for the consistency of the DPF estimator for the expectation under the unnormalized posterior (Eq. \ref{eq:dpf-unbiased})
\begin{align}
    & \lim_{N \rightarrow \infty} \expect{
    \hat{Z} \sum_{i=1}^N \normalized{w}_T^i \left( f_{\theta}(\lineage{x}_{1:T}^i) \grad_{\theta} \log p_{\theta}(\lineage{x}_{1:T}^i, y_{1:T}) + \grad_{\theta} f_{\theta}(\lineage{x}_{1:T}^i) \right)
    }{PF} = \\
    & \expect{p_{\theta}(y_{1:T}) \left( f_{\theta}(x_{1:T}) \grad_{\theta} \log p_{\theta}(x_{1:T}, y_{1:T}) + \grad_{\theta} f_{\theta}(x_{1:T}) \right)}{x_{1:T} \sim p_{\theta}(x_{1:T} | y_{1:T})} = \\
    & \expect{p_{\theta}(y_{1:T}) \left( f_{\theta}(x_{1:T}) \grad_{\theta} \log p_{\theta}(x_{1:T} | y_{1:T}) + \grad_{\theta} f_{\theta}(x_{1:T}) \right) + f_{\theta}(x_{1:T}) \grad_{\theta} p_{\theta}(y_{1:T})}{x_{1:T} \sim p_{\theta}(x_{1:T} | y_{1:T})} = \\
    & \expect{p_{\theta}(y_{1:T}) f_{\theta}(x_{1:T}) \grad_{\theta} \log p_{\theta}(x_{1:T} | y_{1:T}) + \grad_{\theta} \left( f_{\theta}(x_{1:T}) p_{\theta}(y_{1:T}) \right)}{x_{1:T} \sim p_{\theta}(x_{1:T} | y_{1:T})} = \\
    & \grad_{\theta} \expect{p_{\theta}(y_{1:T}) f_{\theta}(x_{1:T})}{x_{1:T} \sim p_{\theta}(x_{1:T} | y_{1:T})} .
\end{align}

\begin{theorem} \label{th:unbiased}
Eq. \ref{eq:supp-expect-dpf} is an unbiased estimator for the expectation under the unnormalized posterior, that is
\begin{align}
    &\expect{
    \hat{Z} \sum_{i=1}^N \normalized{w}_T^i \left( f_{\theta}(\lineage{x}_{1:T}^i) \grad_{\theta} \log p_{\theta}(\lineage{x}_{1:T}^i, y_{1:T}) + \grad_{\theta} f_{\theta}(\lineage{x}_{1:T}^i) \right)
    }{x_t^i \sim q_{\phi}(x_t | x_{t-1}^{a_t^i}, y_t), \ a_t^{1:N} \sim R(\normalized{w}_{t-1}^{1:N})} = \\
    &\grad_{\theta} \expect{p_{\theta}(y_{1:T}) f_{\theta}(x_{1:T})}{x_{1:T} \sim p_{\theta}(x_{1:T} | y_{1:T})} .
\end{align}
\end{theorem}
\begin{proof}
The proof is by induction on $T$ and the inductive hypothesis is precisely the theorem for any function $f$.

The inductive step is as follows. Let $g(x_{1:T+1}) = f_{\theta}(x_{1:{T+1}}) \grad_{\theta} \log p_{\theta}(x_{1:{T+1}}, y_{1:{T+1}}) + \grad_{\theta} f_{\theta}(x_{1:{T+1}})$ and $h(x_{1:T}) = \expect{p_{\theta}(y_{T+1} | x_T^{i}) f_{\theta}(x_{1:{T+1}})}{x_{T+1} \sim p_{\theta}(x_{T+1} | y_{T+1}, x_T^{i})}$. We first show that
\begin{align}
    &\expect{p_{\theta}(y_{T+1} | x_T) g(x_{1:T+1})}{x_{T+1} \sim p_{\theta}(x_{T+1} | y_{T+1}, x_T)} = \\
    &\expect{p_{\theta}(y_{T+1} | x_T) \left( f_{\theta}(x_{1:{T+1}}) \grad_{\theta} \log p_{\theta}(x_{1:{T+1}}, y_{1:{T+1}}) + \grad_{\theta} f_{\theta}(x_{1:{T+1}}) \right)}{x_{T+1} \sim p_{\theta}(x_{T+1} | y_{T+1}, x_T)} = \\
    & \expect{p_{\theta}(y_{T+1} | x_T) f_{\theta}(x_{1:{T+1}})}{x_{T+1} \sim p_{\theta}(x_{T+1} | y_{T+1}, x_T)}
        \grad_{\theta} \log p_{\theta}(x_{1:{T}}, y_{1:{T}}) + \\
    &\expect{p_{\theta}(y_{T+1} | x_T) \left( f_{\theta}(x_{1:{T+1}}) \grad_{\theta} \log p_{\theta}(x_{{T+1}}, y_{{T+1}} | x_T) + \grad_{\theta} f_{\theta}(x_{1:{T+1}}) \right)}{x_{T+1} \sim p_{\theta}(x_{T+1} | y_{T+1}, x_T)} = \\
    & h(x_{1:T}) \grad_{\theta} \log p_{\theta}(x_{1:{T}}, y_{1:{T}}) +
    \mathop{\mathbb{E}}_{x_{T+1}} \Bigg[ p_{\theta}(y_{T+1} | x_T) f_{\theta}(x_{1:{T+1}}) \grad_{\theta} \log p_{\theta}(x_{{T+1}} | y_{{T+1}}, x_T) + \nonumber \\
        &\quad\quad p_{\theta}(y_{T+1} | x_T) f_{\theta}(x_{1:{T+1}}) \frac{\grad_{\theta} p_{\theta}(y_{T+1} | x_T)}{p_{\theta}(y_{T+1} | x_T)} + 
        p_{\theta}(y_{T+1} | x_T) \grad_{\theta} f_{\theta}(x_{1:{T+1}}) \Bigg] = \\
    & h(x_{1:T}) \grad_{\theta} \log p_{\theta}(x_{1:{T}}, y_{1:{T}}) +
    \mathop{\mathbb{E}}_{x_{T+1} \sim p_{\theta}(x_{T+1} | y_{T+1}, x_T)} \Big[ p_{\theta}(y_{T+1} | x_T) f_{\theta}(x_{1:{T+1}}) \grad_{\theta} \log p_{\theta}(x_{{T+1}} | y_{{T+1}}, x_T) + \nonumber \\
        &\quad\quad \grad_{\theta} \left( p_{\theta}(y_{T+1} | x_T)  f_{\theta}(x_{1:{T+1}}) \right) \Big] = \\
    & h(x_{1:T}) \grad_{\theta} \log p_{\theta}(x_{1:{T}}, y_{1:{T}}) + 
    \grad_{\theta} \expect{p_{\theta}(y_{T+1} | x_T) f_{\theta}(x_{1:{T+1}})}{x_{T+1} \sim p_{\theta}(x_{T+1} | y_{T+1}, x_T)} = \\
    & h(x_{1:T}) \grad_{\theta} \log p_{\theta}(x_{1:{T}}, y_{1:{T}}) + \grad_{\theta} h(x_{1:T}) .
\end{align}

Then
\begin{align}
    &\expect{
    \left( \prod_{t=1}^{T+1} W_t \right) \sum_{i=1}^N \normalized{w}_{T+1}^i \left( f_{\theta}(\lineage{x}_{1:{T+1}}^i) \grad_{\theta} \log p_{\theta}(\lineage{x}_{1:{T+1}}^i, y_{1:{T+1}}) + \grad_{\theta} f_{\theta}(\lineage{x}_{1:{T+1}}^i) \right)
    }{x_t^{1:N}, a_t^{1:N}} = \\
    &\expect{
    \left( \prod_{t=1}^{T+1} W_t \right) \sum_{i=1}^N \normalized{w}_{T+1}^i g(\lineage{x}_{1:T+1}^i)
    }{x_t^{1:N}, a_t^{1:N}} = \\
    &\expect{ \left( \prod_{t=1}^{T} W_t \right)
        \expect{W_{T+1} \sum_{i=1}^N \normalized{w}_{T+1}^i g(\lineage{x}_{1:T+1}^i)}{a_{T+1}^{1:N}, x_{T+1}^{1:N}}}{x_{1:T}^{1:N}, a_{1:T}^{1:N}} = \\
    &\expect{ \left( \prod_{t=1}^{T} W_t \right)
        \expect{\expect{\sum_{i=1}^N w_{T+1}^i g(\lineage{x}_{1:T+1}^i)}{x_{T+1}^{1:N}}}{a_{T+1}^{1:N}}}{x_{1:T}^{1:N}, a_{1:T}^{1:N}} = \\
    &\expect{ \left( \prod_{t=1}^{T} W_t \right)
        \expect{\frac{1}{N}\sum_{i=1}^N  \expect{\frac{p_{\theta}(x_{T+1}, y_{T+1} | x_T^{a_{T+1}^i})}{q_{\phi}(x_{T+1} | x_T^{a_{T+1}^i})} g(\lineage{x}_{1:T}^{a_{T+1}^i}, x_{T+1})}{x_{T+1} \sim q_{\phi}(x_{T+1} | x_T^{a_{T+1}^i})}}{a_{T+1}^{1:N}}}{x_{1:T}^{1:N}, a_{1:T}^{1:N}} = \\
    &\expect{ \left( \prod_{t=1}^{T} W_t \right)
        \expect{\frac{1}{N}\sum_{i=1}^N  \expect{p_{\theta}(y_{T+1} | x_T^{a_{T+1}^i}) g(\lineage{x}_{1:T}^{a_{T+1}^i}, x_{T+1})}{x_{T+1} \sim p_{\theta}(x_{T+1} | y_{T+1}, x_T^{a_{T+1}^i})}}{a_{T+1}^{1:N}}}{x_{1:T}^{1:N}, a_{1:T}^{1:N}} = \\
    &\expect{ \left( \prod_{t=1}^{T} W_t \right)
        \sum_{i=1}^N \normalized{w}_T^i \expect{p_{\theta}(y_{T+1} | x_T^{i}) g(\lineage{x}_{1:T}^{i}, x_{T+1})}{x_{T+1} \sim p_{\theta}(x_{T+1} | y_{T+1}, x_T^{i})}}{x_{1:T}^{1:N}, a_{1:T}^{1:N}} = \\
    &\expect{ \left( \prod_{t=1}^{T} W_t \right)
        \sum_{i=1}^N \normalized{w}_T^i \left( h(x_{1:T}) \grad_{\theta} \log p_{\theta}(x_{1:{T}}, y_{1:{T}}) + \grad_{\theta} h(x_{1:T}) \right)}{x_{1:T}^{1:N}, a_{1:T}^{1:N}} = \\
    &\grad_{\theta} \expect{p_{\theta}(y_{1:T}) h(x_{1:T})}{x_{1:T} \sim p_{\theta}(x_{1:T} | y_{1:T})} =
    \grad_{\theta} \expect{p_{\theta}(y_{1:T+1}) f_{\theta}(x_{1:T+1})}{x_{1:T+1} \sim p_{\theta}(x_{1:T+1} | y_{1:T+1})} ,
\end{align}
where the penultimate equality is the inductive hypothesis for $T$ and $h$. This completes the proof.

\end{proof}

\section{Additional experiments}

In Table \ref{tab:bouncing_ball} we provide numerical results for the experiments presented in the main text. We also include an additional experiment, with a comparison to the differentiable particle filter based on optimal transport (DPF-OT) proposed by \citet{corenflos_differentiable_2021}.

\begin{table}
    \centering
    \begin{tabular}{llr}
            \toprule
            N   &     Method   &  LGSSM \\
            \midrule
                  & PF-SF           &       -475.73 $\pm$ 55.69  \\
           10     & SIS             &       -2742.99 $\pm$ 279.97 \\
                  & PF              &       -300.46 $\pm$ 2.62 \\
                  & DPF-SGR (us)    &       \textbf{-292.70 $\pm$ 1.28} \\
            \midrule
                  &   PF-SF              &       -1038.85 $\pm$ 176.02\\
            10k   &   SIS                &       -343.65 $\pm$ 15.61\\
                  &   PF                 &       \textbf{-290.04 $\pm$ 0.00}\\
                  &   DPF-SGR (us)       &       -290.08 $\pm$ 0.02 \\
            \bottomrule
        \end{tabular}
    \quad
    \begin{tabular}{llr}
        \toprule
        N     &     Method   & VRNN \\
        \midrule
                & --   & \multicolumn{1}{c}{--} \\
         10     & SIS           &       -356.22 $\pm$ 0.08 \\
               & PF            &       -355.98 $\pm$ 0.13  \\
              & DPF-SGR (us)  &       \textbf{-355.82 $\pm$ 0.14} \\
        \midrule
                & -- &                  \multicolumn{1}{c}{--}        \\
        20    & SIS              &       -356.41 $\pm$ 0.50\\
              & PF               &       -355.97 $\pm$ 0.16\\
              & DPF-SGR (us)     &       \textbf{-355.74 $\pm$ 0.02} \\
        \bottomrule
    \end{tabular}
    \vspace{5mm}
    \caption{Best values for the test set log-marginal likelihood $\log Z$ for $N$ particles in the LGSSM/VRNN experiments. }
    \label{tab:bouncing_ball}
\end{table}

\subsection{Stochastic Volatility Model} \label{sec:volatility}

\begin{figure}[t]
    \centering
    \includegraphics[width=0.49\textwidth]{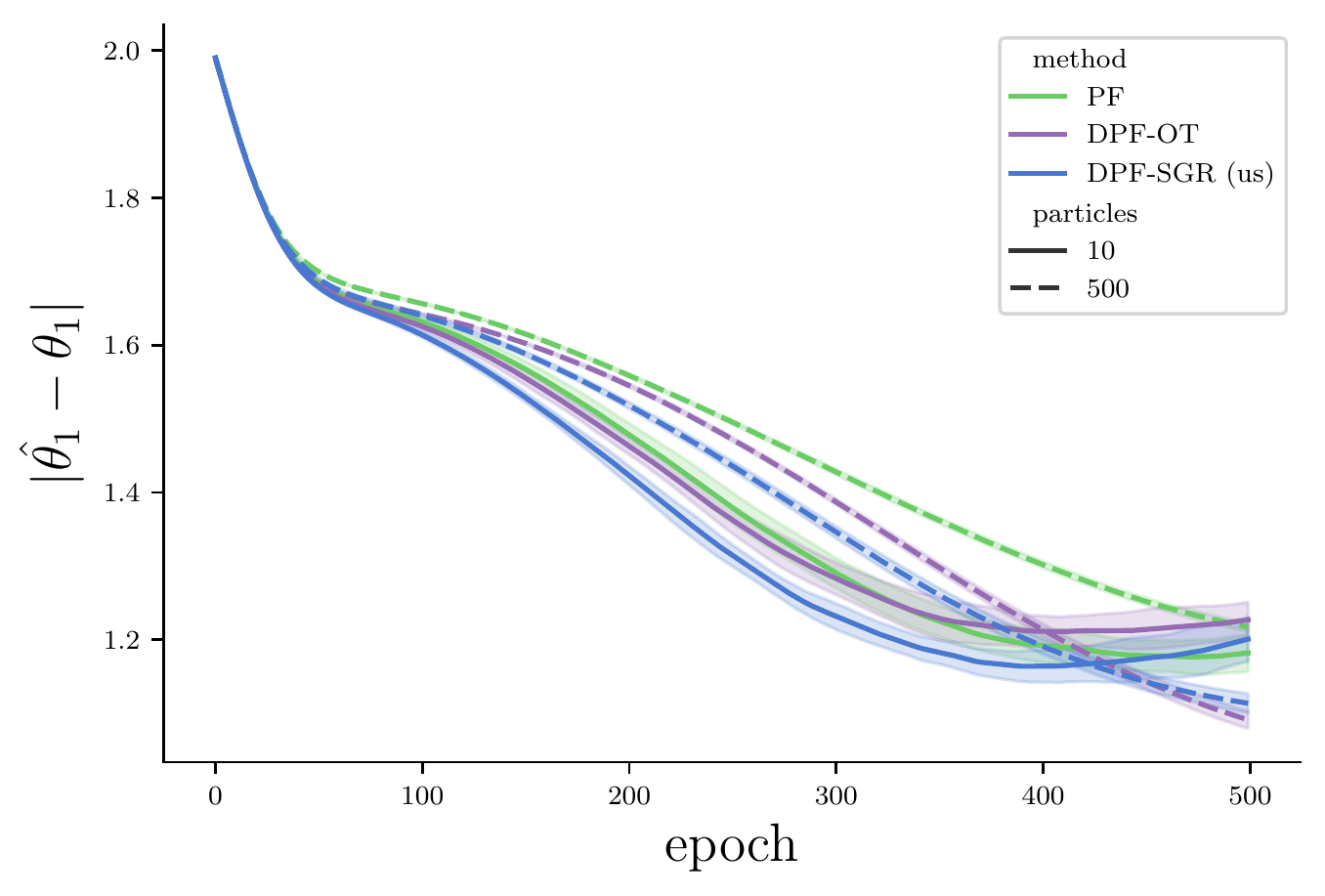}
    \includegraphics[width=0.49\textwidth]{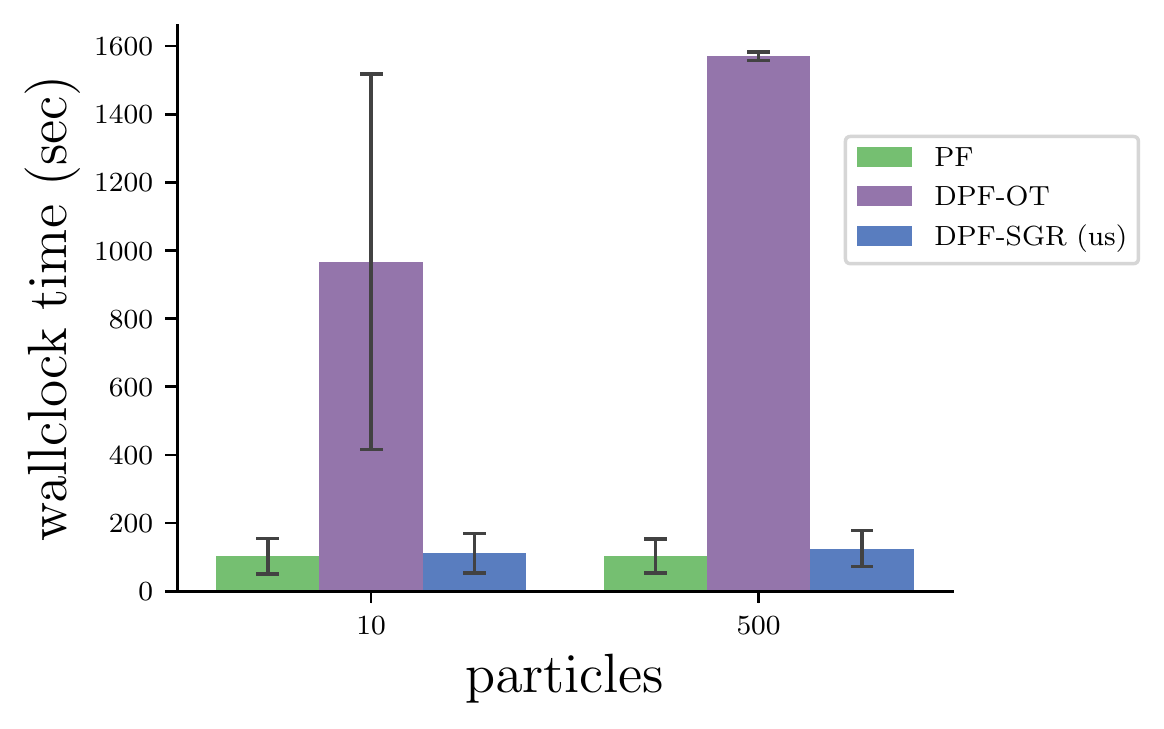}
    \caption{Left is the L1 error in parameter estimation throughout the training of the stochastic volatility model. Right is the average wall clock execution time for a single forward and backward pass for each particle filter variant.}
    \label{fig:experiments/stochastic_volitility}
\end{figure}

In this experiment we learn the parameters of a stochastic volatility model, a popular financial model for which particle filters are known to perform well. We use the implementation provided in the code repository associated with the paper of \citet{corenflos_differentiable_2021}\footnote{\url{https://github.com/JTT94/filterflow}}, which allows us to reuse the DPF-OT settings provided by the original authors. The model is described by the following equations:
\begin{align*}
x_0 &\sim \mathcal{N}\left(0, \frac{\sigma^2_x}{(1 -  \phi^2)}\right)\\
    x_{t+1} &= \mu\cdot(1 - \phi) + \phi\cdot x_t + \sigma_x \cdot \eta_t,\\
    y_t &= \epsilon_t \cdot  \exp\left(\frac{x_t}{2}\right),\\
    \eta_t &\sim \mathcal{N}(0,1), t = 1,..., T,\\
    \epsilon_t &\sim \mathcal{N}(0,1) .
\end{align*}

We generate synthetic data with T = 100, $\mu$ = 2, $\phi$ = 0.9, $\sigma_x$ = 1, $\theta_{1:3} = \{\mu,\phi, \sigma_x \}$ and train the model for 500 epochs using Adam with a learning rate of $0.01$. Resampling is triggered when the effective sample size falls beneath $N/2$, and for all methods we use the bootstrap proposal like in the LGSSM experiment. Figure \ref{fig:experiments/stochastic_volitility} shows progression of the parameter estimation error $|\hat{\mu} - \mu|$ throughout training and the execution time of different algorithms. DPF-SGR produces results at least on par with DPF-OT, while executing almost as fast as PF.

\section{Additional background}

Particle filters \citep{doucet2009tutorial,chopin_introduction_2020,doucet2001introduction} are an extremely popular family of algorithms for state estimation in non-linear state space models. Their versatility, ease of implementation, and well-understood theoretical behavior make them a popular choice in a wide variety of application domains, including robotics, ecology, and finance. The crucial advantage of particle filters is that they can adaptively focus computation on more promising trajectories, which is usually accomplished through discrete resampling steps. Since particle filters also provide estimates of the log-marginal likelihood, they can be used to facilitate model learning.

It has been known for a long time \citep{wengert_simple_1964, griewank_automatic_1989} that programs computing derivatives can be automatically obtained from programs computing differentiable functions, the corresponding algorithms being known as ``automatic differentiation'' (AD). A specialized version of this technique is used in deep learning under the name ``backpropagation'' \citep{rumelhart_learning_1986}, more recently becoming popular across all of machine learning \citep{baydin_automatic_2018}, accelerated by the availability of open-source libraries such as PyTorch \citep{paszke_pytorch_2019} and TensorFlow \citep{tensorflow2015-whitepaper}. In recent years there has been an explosion of interest in the machine learning community in computing gradients of expressions involving sampling random variables, motivated primarily by problems in deep generative modeling \citep{neal_view_1998} and reinforcement learning \citep{williams_simple_1992}. While generic methods for extending AD to handle stochastic computations have been developed \citep{schulman_gradient_2016, foerster_dice_2018}, the variance of resulting gradient estimators is still problematic\citep{rezende_stochastic_2014} where the reparameterization trick \citep{kingma_auto-encoding_2014, ranganath_black_2013} does not apply, in particular for discrete variables.

Building on those developments, it is becoming increasingly popular to learn sequential latent variable models by gradient ascent using estimators derived from the application of AD to the log-marginal likelihood estimators obtained with particle filters. Examples include robotic localization \citep{jonschkowski_differentiable_2018}, world-modeling in reinforcement learning \citep{igl2018deep}, and deep generative modeling of music and speech \citep{maddison_filtering_2017}. The main obstacle is the discrete resampling steps, for which existing gradient estimators suffer from excessive variance, leading to the dependence of the resampling probabilities on model parameters typically being ignored \citep{le_auto-encoding_2018, maddison_filtering_2017, naesseth_variational_2018}, introducing bias that was shown not to vanish asymptotically \citep{corenflos_differentiable_2021}.

Several recent articles\citep{corenflos_differentiable_2021, zhu_towards_2020, karkus_particle_2018} have proposed continuous relaxations of particle filter discrete resampling steps to avoid this problem, with different trade-offs and varying degrees of complexity. 

Alternatively, there exist consistent estimators of the gradient of log-marginal likelihood that can be obtained with an unmodified pass of a particle filter with usual discrete resampling \citep{poyiadjis_particle_2011}, but until now these could not be obtained by applying AD to the log-marginal likelihood estimator produced by a particle filter.

\subsection{Particle filters} \label{sec:supp-pf}

Here we briefly review the basic particle filter algorithm, summarized in Algorithm 1, referring the reader to \citep{chopin_introduction_2020,doucet2009tutorial,doucet2001introduction} for more details. For a given sequence of observations $y_{1:T}$, particle filters approximate the filtering distributions $p_{\theta}(x_t | y_{1:t})$ and the marginal likelihood $p_{\theta}(y_{1:T})$ of a state space model with the following factorization of the joint
\begin{align}
    p_{\theta}(x_{0:T}, y_{1:T}) = p_{\theta}(x_0) \prod_{t=1}^T p_{\theta}(x_t | x_{t-1}) p_{\theta}(y_t | x_t).
\end{align}

The algorithm starts by sampling $N$ particles independently from an initial distribution $x_0^i \sim p_{\theta}(x_0)$ and setting their weights to be uniform $w_0^i = \frac{1}{N}$. Then, at each iteration it advances one step forward in time, by first advancing each particle according to a proposal distribution $q_{\phi}(x_{t+1} | x_t)$, then accumulating the importance sampling ratio $\frac{p_{\theta}(x_{t+1} | x_t)}{q_{\phi}(x_{t+1} | x_t)} $ and the likelihood $p_{\theta}(y_{t+1} | x_{t+1})$ in the corresponding weight, saving the sum of weights and normalizing them before proceeding to the next time step. We assume that $q$ and $p$ have the same support and that the likelihoods are always non-zero, which guarantees that all the weights are always positive. We also assume that both $p$ and $q$ are differentiable with respect to their parameters.

As the weights accumulate more terms, they become increasingly unbalanced and it is more efficient to focus more computational resources to particles with higher weights. This is accomplished via a resampling step, which in our formulation is optionally performed at the beginning of each time step. The ancestral indices for each particle are chosen from a resampling distribution $R(\normalized{w}_{t-1}^{1:N})$, which in the simplest multinomial resampling picks ancestral indices independently from a categorical distribution, but other variants like stratified and systematic resampling are usually favored in practice \citep{murray2016parallel}. We assume post-resampling weights are always uniform for simplicity, although non-uniform resampling schemes also exist \citep{fearnhead_-line_2003}. Note that we write $\resampled{w}$ for the weights obtained after resampling, which in a standard particle filter are always $\frac{1}{N}$, $w$ for weights after the likelihood was accumulated at a given time step, and $\normalized{w}$ for the subsequently normalized weights.

Regardless of the specific details of the particle filter algorithm used, the marginal likelihood $Z = p_{\theta}(y_{1:T})$ is approximated by $\hat{Z} = \prod_{t=1}^T W_t$. This estimator is known to be unbiased and under weak assumptions consistent. The estimator for the log-marginal likelihood $\log p_{\theta}(y_{1:T})$ given by $\log \hat{Z}$ is consistent but biased. Similarly, a particle filter produces consistent but biased estimators for expectations under the posterior, while the estimators for the expectations under the unnormalized posterior are additionally unbiased. We refer the reader to \citet{moral_feynman-kac_2004} and \citet{chopin_introduction_2020} for detailed convergence results.

\subsection{Score function estimation with particle filters}

Maximum likelihood learning of generative models by gradient ascent can be performed using the
score function, which is the gradient of the log-marginal likelihood $\grad_{\theta} \log p_{\theta}(y_{1:T})$. While the exact value of the score function is in general intractable, it can be approximated using any method for sampling from the posterior distribution with the help of Fisher's identity
\begin{align}
    &\grad_{\theta} \log p_{\theta}(y_{1:T}) =
    \grad_{\theta} \log \expect{p_{\theta}(x_{1:T}, y_{1:T})}{x_{1:T} \sim p(x_{1:T})} = \\
    &\expect{\grad_{\theta} \log p_{\theta}(x_{1:T}, y_{1:T})}{x_{1:T} \sim p_{\theta}(x_{1:T} | y_{1:T})} .
\end{align}

In the particle filtering context, the approximate samples from the posterior can be obtained by tracing the ancestral lineages of the surviving particles form the final time step. Let $\lineage{x}_{1:t}^i$ be the ancestral lineage of a particle $i$ at time $t$, that is $\lineage{x}_{1:t+1}^i = (\lineage{x}_{1:t}^{a_{t+1}^i}, x_{t+1}^i)$ and $\lineage{x}_1^i = x_1^i$. \citet{poyiadjis_particle_2011} propose a consistent score function estimator of the form
\begin{align}
    \grad_{\theta} \log p_{\theta}(y_{1:T}) \approx \sum_{i=1}^N \normalized{w}_T^i \grad_{\theta} \log p_{\theta}(\lineage{x}_{1:T}^i, y_{1:T}) .
    \label{eq:supp-poyadjis}
\end{align}
They also show how to compute this estimator without having to store full ancestral lineages and that the variance of this estimator can be further reduced using the marginal particle filter. 

In this work we show how to obtain this estimator through automatic differentiation of $\log \hat{Z}$ produced by a particle filter, without modifying the estimators obtained on the forward pass.

\subsection{The stop-gradient operator} \label{sec:supp-stop-gradient}

The stop-gradient operator is a standard component of automatic differentiation libraries, called \lstinline{stop_gradient} in Tensorflow and \lstinline{detach} in PyTorch, producing an expression that evaluates to its normal value on the forward pass but has zero gradient. While stop-gradient does not correspond to any formal mathematical function, we need to include it in our formulas to formally reason about programs involving stop-gradient and how automatic differentiation acts upon them. Following \citet{foerster_dice_2018}, we denote the stop-gradient with a $\bot$, for example $\grad \stopgrad{x} = 0$.

Expressions involving the stop-gradient operator should be understood as expressions in a formal calculus, which correspond to programs written using automatic differentiation libraries. Following established terminology, we sometimes refer to them as \emph{surrogate expressions}. For such expressions we need to distinguish between the results obtained by evaluating them on the forward pass and the expressions obtained by applying the gradient operator. Following \citet{van_krieken_storchastic_2021}, we denote evaluation of an expression $E$ with an overhead right arrow $\evaluate{E}$. Operationally, evaluation removes all instances of stop-gradient, provided that no gradient operator is acting on them. Gradients satisfy the usual chain rules, except that gradients of any expression wrapped in stop-gradient are zero. Formally, the key equations are:
\begin{align}
    \evaluate{f(E_1, \dots, E_n)} &= f(\evaluate{E_1}, \dots, \evaluate{E_2}) \\
    \evaluate{\stopgrad{E}} &= \evaluate{E} \\
    \grad{\stopgrad{E}} &= 0 \\
    \evaluate{\grad{E}} &= \grad{E} \quad \text{if} \ \bot \notin E  ,
\end{align}
where $f$ is some differentiable function and $\bot \notin E$ means that stop-gradient is not applied anywhere in expression $E$, including its subexpressions. We use this calculus to derive formulas for the estimators obtained by algorithms utilizing stop-gradient in automatic differentiation, which is generally accomplished by using the rules above, along with the chain rule for the gradient operator, to push the gradients inside, past any stop-gradient operators, and then evaluating the resulting expression. Note that in this calculus the gradient operator produces expressions rather than values, so that we can model repeated automatic differentiation. Expressions without stop-gradient can be equated with their usual interpretations. Throughout the paper we are careful to avoid using integrals of expressions involving stop-gradient, as we consider those undefined. We also present a short tutorial on using this calculus in the appendix.

In previous work \citep{foerster_dice_2018, van_krieken_storchastic_2021}, the stop-gradient operator has been used to show that certain expressions produce unbiased estimators for derivatives of arbitrary order under repeated automatic differentiation. In this work we consider a different family of expressions involving stop-gradient, showing that under automatic differentiation they produce score function estimators corresponding to the application of Fisher's identity and self-normalizing importance sampling.

\subsection{Automatic differentiation of stochastic computation}

When using AD for computation involving stochastic choices, it is typically desirable to obtain gradient estimators that in expectation equal the gradient of the expected value of the original computation. The generic methods for accomplishing this generally fall into two categories. Consider a computation with a single random variable, $\mathcal{L} = \expect{f_{\theta}(x)}{x \sim p_{\theta}(x)}$. If $p_{\theta}$ is a continuous distribution, we can express $x$ as a deterministic function $h$ of some random variable $\epsilon$ with a fixed distribution independent of $\theta$, such as $\mathcal{N}(0,1)$. With that, the gradient commutes with the expectation
\begin{align}
    \grad_{\theta} \mathcal{L} =
    \grad_{\theta} \expect{f_{\theta}(h(\epsilon, \theta))}{\epsilon \sim \mathcal{N}(0,1)} =
    \expect{\grad_{\theta} f_{\theta}(h(\epsilon, \theta))}{\epsilon \sim \mathcal{N}(0,1)} ,
\end{align}
at which point AD can be applied for a fixed $\epsilon$. This technique is broadly known as the pathwise derivative \citep{glasserman_monte_2003}, while in the machine learning it's usually being referred to as the reparameterization trick \citep{kingma_auto-encoding_2014, ranganath_black_2013}. It is easy to apply for some distributions such as Gaussian and uniform, but in general finding a suitable $\epsilon$ and $h$ that can be efficiently computed is difficult and an active research area \citep{jankowiak_pathwise_2018, graves_stochastic_2016}. When applicable, this method tends to produce gradient estimators with relatively low variance.

An alternative, typically used when reparemeterization is not possible, is based on the following identity
\begin{align}
    \grad_{\theta} \expect{f_{\theta}(x)}{x \sim p_{\theta}(x)} =
    \expect{\grad_{\theta} f_{\theta}(x) + f_{\theta}(x) \grad_{\theta} \log p_{\theta}(x)}{x \sim p_{\theta}(x)} .
\end{align}
The two gradients inside the expectation are typically computed with AD jointly for a fixed $x$ by forming a surrogate expression $f_{\theta}(x) + \stopgrad{f_{\theta}(x)} \log p_{\theta}(x)$. In the reinforcement learning literature this method is known as REINFORCE \citep{williams_simple_1992}, and more broadly referred to as the score function method or the likelihood ratio method \citep{fu_chapter_2006}. We refrain from using those terms to avoid confusion with the score function $\grad_{\theta} \log p_{\theta}(y_{1:T})$ and the likelihood $p_{\theta}(y_{1:T})$ associated with the sequential latent variable model. Instead, we refer to this method as DiCE \citep{foerster_dice_2018}, which produces the same gradient estimator using the following surrogate expression
\begin{align}
    &\evaluate{\grad_{\theta} \left( \frac{p_{\theta}(x)}{\stopgrad{p_{\theta}(x)}} f_{\theta}(x) \right)} =\\
    &f_{\theta}(x) \evaluate{\frac{\grad_{\theta} p_{\theta}(x)}{\stopgrad{p_{\theta}(x)}}} + \evaluate{\frac{p_{\theta}(x)}{\stopgrad{p_{\theta}(x)}}} \grad_{\theta} f_{\theta}(x) = \\
    &f_{\theta}(x) \grad_{\theta} \log p_{\theta}(x) + \grad_{\theta} f_{\theta}(x) ,
\end{align}
while also remaining unbiased under repeated AD. Incidentally, the term $\frac{p_{\theta}(x)}{\stopgrad{p_{\theta}(x)}}$, which \citet{foerster_dice_2018} call ``the magic box'', is precisely the AD-centric expression of the likelihood ratio from which the likelihood ratio method derives its name. The rules for constructing the surrogate expression in the presence of multiple random variables and complex dependencies were given by \citet{schulman_gradient_2016}, with \citet{foerster_dice_2018} deriving equivalent rules for DiCE. While DiCE is broadly applicable, the resulting gradient estimators tend to have high variance, which to some extent can be mitigated through a careful use of baselines \citep{mao_baseline_2019, farquhar_loaded_2019}.

\end{document}